\documentclass{article}
\usepackage{iclr2027_conference,times}
\usepackage[T1]{fontenc}
\usepackage[utf8]{inputenc}
\usepackage{amsmath,amssymb,amsthm,mathtools}
\usepackage{graphicx,booktabs,tabularx,adjustbox}
\usepackage{algorithm,algorithmic}
\usepackage{placeins}
\usepackage{flafter}
\usepackage{enumitem}
\usepackage{xcolor}
\usepackage{hyperref}
\usepackage{url}
\hypersetup{colorlinks=true,citecolor=blue,linkcolor=blue,urlcolor=blue}
\newtheorem{proposition}{Proposition}
\newtheorem{theorem}{Theorem}
\newtheorem{assumption}{Assumption}
\newtheorem{definition}{Definition}
\DeclareMathOperator{\KL}{KL}
\DeclareMathOperator{\softmax}{softmax}
\DeclareMathOperator{\ESS}{ESS}
\DeclareMathOperator{\sign}{sign}
\newcommand{\Rmain}{R_{\mathrm{main}}}
\newcommand{\PTV}{P_{\mathrm{TV}}}
\newcommand{\KLenv}{\KL_{\mathrm{env}}}

\title{Environment-Conditioned Tail Reweighting\\for Invariant Learning under Mixed Shifts}
\author{Yuanchao Wang$^{1}$ \quad Tianqi Zhong$^{2}$ \quad Fengnan Li$^{1}$ \quad Zhao-Rong Lai$^{3}$\\[0.5em]
\normalfont $^{1}$Duke University \quad $^{2}$Beihang University\\
\normalfont $^{3}$Guangzhou University}
\iclrfinalcopy
\hypersetup{
  pdftitle={Environment-Conditioned Tail Reweighting for Invariant Learning under Mixed Shifts},
  pdfauthor={Yuanchao Wang, Tianqi Zhong, Fengnan Li, Zhao-Rong Lai}
}
\begin{document}
\maketitle
\lhead{Preprint}
\begin{abstract}
Out-of-distribution generalization becomes challenging when spurious correlations vary across environments while difficult or underrepresented samples remain insufficiently emphasized within them. Invariant learning and reweighting address complementary aspects of this mixed-shift problem, yet applying reweighting only to prediction leaves the invariance constraint evaluated on a different training risk. We introduce Environment-Conditioned Tail Reweighting (ECTR), which couples cross-environment TV invariance with within-environment adversarial tail weighting through a shared weighted risk. Sample weights are normalized within each environment, and the resulting risk is used consistently for both predictive learning and the total-variation (TV) stationarity penalty, while an environment-wise KL term controls adversarial concentration. Our analysis characterizes the stationarity-dependent weighting signal and its KL-regularized distributional interpretation, alongside conditional-risk and optimization properties under the stated assumptions. Controlled mixed-shift ablations consistently support the shared-risk coupling across three difficulty settings. Across synthetic and real-world benchmarks, ECTR achieves favorable results under both observed- and inferred-environment settings, including improvements over TV-based parent methods and strong performance against additional invariant-learning baselines. Training-corruption experiments further show that ECTR assigns less excess weight to corrupted samples than the tested fixed-tail rules. Together, these results provide a unified theoretical and empirical account of environment-conditioned tail reweighting for invariant learning under mixed shifts.
\end{abstract}

\section{Introduction}
\label{sec:intro}
Deployment shifts often combine two difficulties. Under \emph{correlation shift}, spurious features change their relationships with labels across environments. Under \emph{diversity shift}, rare or underrepresented variations in style, context, or population become important at test time \citep{ye2022oodbench}. These difficulties can coexist within the same task: a training environment may exhibit a strong spurious correlation while containing only a small number of difficult examples that contradict it. A predictor must then generalize across environments without overlooking useful variation within them.
The two challenges act at different levels. Comparing training environments can reveal that a predictive association is unstable, but each environment average is still dominated by its prevalent examples. Conversely, increasing the importance of difficult examples can expose underrepresented variation without ensuring that the resulting predictor is stable across environments. Mixed shift therefore calls for more than either an across-environment constraint or a within-environment emphasis in isolation. The risk used to express one should be consistent with the risk used to express the other.

Invariant risk minimization (IRM) and total-variation (TV) formulations encourage a common predictor across environments \citep{arjovsky2019irm,lai2024tvirm,wang2025oodtvirm}. Their constraints target correlation shift, but stationarity is usually measured on environment-average risks. A small average gradient can conceal heterogeneous sample behavior. Reweighting and distributionally robust optimization instead emphasize difficult samples \citep{sagawa2020groupdro,duchi2021uniformdro,liu2021jtt}. However, global weighting can also change the relative importance of environments, while weighting only the predictive loss leaves the invariance test on the original, unweighted risks. Neither modification explicitly aligns the two components.

We introduce \emph{Environment-Conditioned Tail Reweighting} (ECTR) to make this alignment explicit. A sample adversary learns global weights, which are then normalized separately within each environment. Coupling changes both the distribution used to evaluate stationarity and the signal that learns that distribution: the adversary responds jointly to predictive loss, stationarity contributions, and KL regularization. The predictor is trained on the same environment-conditioned weighted risks, with the outer environment average fixed.
The distinction is important even when both methods include TV regularization. A \emph{Main-only} alternative weights prediction but computes TV on uniform environment risks. Full ECTR instead computes TV on the weighted risks themselves. In controlled LOGIT-Z experiments that combine correlation reversal and increased feature-noise variance, Full improves worst-group accuracy over Main-only in all three difficulty settings. The comparison supports aligning prediction and stationarity through the shared weighted risk.
ECTR controls concentration through a negative environment-wise KL term in the adversarial objective. This term regulates how strongly the adversary emphasizes a small set of examples. We evaluate its behavior through predictive performance, noise mass under training corruption, and concentration across adversary settings. Observed environments define the core method; soft memberships inferred from auxiliary variables extend it to learning without supplied training environment labels. Shared-membership experiments compare ECTR with TV-IRM under matched grouping information.

Our contributions are threefold:
\begin{itemize}[leftmargin=*,labelindent=0pt]
\item \textbf{Shared-risk coupling.} We couple sample-weight learning and TV stationarity through the same environment-conditioned risk used for prediction.
\item \textbf{Concentration-controlled formulation and analysis.} We characterize the stationarity-dependent weighting signal and its KL-regularized distributional interpretation, alongside conditional-risk and optimization properties under explicit assumptions.
\item \textbf{Empirical evaluation.} We demonstrate coupling benefits through controlled ablations and comparisons with TV-based parents. Separate studies examine simple tails, corruption, environment information, stability, and cost.
\end{itemize}

Figure~\ref{fig:overview} follows samples through the predictor and sample adversary into a shared conditional risk, which feeds both predictive learning and TV stationarity. Section~\ref{sec:method} defines these connections and the resulting minimax objective.

\section{Related work}
\label{sec:related}
\paragraph{Invariant learning constrains behavior across environments.}
IRM seeks representations supporting a common predictor across environments \citep{arjovsky2019irm,ahuja2020irmgames}. TV-based formulations provide the stationarity mechanism adopted here \citep{lai2024tvirm,wang2025oodtvirm}. Other approaches regulate risk variation: V-REx penalizes risk variance \citep{krueger2021VREx}, RDM matches risk distributions \citep{nguyen2024RDM}, and WRI uses reweighting to impose risk invariance under invariant-feature shift \citep{wong2024WRI}. ECTR builds on TV-IRM by changing the within-environment risk used by both prediction and stationarity.

\paragraph{Reweighting changes which training examples drive learning.}
Group DRO emphasizes difficult groups \citep{sagawa2020groupdro}; sample-level DRO and stable learning reshape the training measure \citep{duchi2021uniformdro,pengcui2020stablereweightSRDO,zhouPengcui2022maple}. Deep stable learning learns sample weights to reduce feature dependencies \citep{pengcui2021deepstable}, while random sample weighting constructs heterogeneous training environments for invariant learning \citep{pengcui2023covariaterandomweight}. LfF and JTT emphasize examples that challenge biased predictors \citep{nam2020LFF,liu2021jtt}. ECTR retains a declared outer environment measure while learning conditional sample weights to emphasize within-environment tails. Its KL interpretation follows entropy-regularized robust optimization, and its contribution is coupling the same weighted risk to prediction and TV stationarity. The simple-tail comparisons evaluate environment-wise Top-$\alpha$ and empirical CVaR weighting~\citep{rockafellar2002CVaR}.

\paragraph{Auxiliary information supports environment inference.}
EIIL, ZIN, and TIVA study learning without supplied partitions \citep{creager2021eiil,lin2022zin,tan2023TIVAtiva}. HRM jointly learns latent heterogeneity and invariant relationships \citep{liu2021hrm}, and XRM discovers environments for invariant learning \citep{pezeshki2024xrm}. ECTR adopts an auxiliary-information-based inference strategy; its contribution lies in the shared weighted-risk construction rather than a new environment-discovery procedure. Its shared-membership evaluation compares objectives under identical environment assignments.

\section{Method}
\label{sec:method}
Figure~\ref{fig:overview} separates weight construction from risk coupling. The predictor maps each input to a prediction and loss, while a second network scores samples for the adversary. These global scores do not directly determine environment importance: conditioning first turns them into a distribution inside each environment. The resulting risk then branches into prediction and TV stationarity. KL acts on the conditional distribution, and a nonnegative multiplier sets the strength of the TV penalty. Only the predictor is needed at test time.

\begingroup
\setlength{\intextsep}{0.6\intextsep}
\setlength{\textfloatsep}{0.6\textfloatsep}
\begin{figure}[!htbp]
\centering
\includegraphics[page=1,width=\linewidth,keepaspectratio]{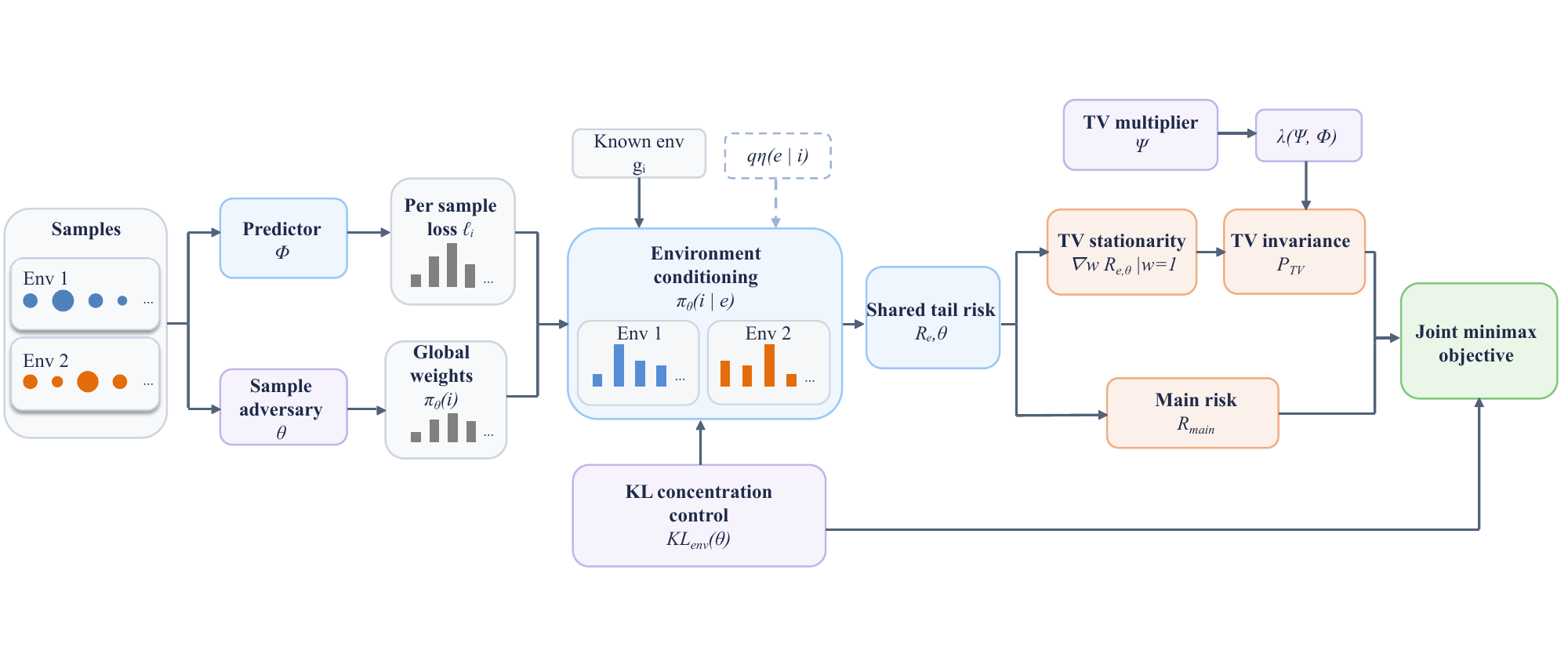}
\caption{ECTR couples environment-conditional tail weighting with TV stationarity through a shared weighted risk. The predictor supplies sample losses and the adversary supplies global weights, which are conditioned on known or inferred environments. The same risk enters the prediction objective and the TV term, while environment-wise KL regularization controls weight concentration.}
\label{fig:overview}
\end{figure}
\endgroup

\subsection{Environment-conditioned weighting}
Let $f_\Phi$ denote the complete trained predictor, including any representation and head. For a minibatch $\{(x_i,y_i)\}_{i=1}^N$, define $\ell_i(w)=\ell(w f_\Phi(x_i),y_i)$. The scalar $w$ is fixed at $1$ and is used only to probe stationarity; it is not a trained head. We suppress evaluation at $w=1$ except when needed. The default outer aggregation is $\mathbb E_e[A_e]=E^{-1}\sum_{e=1}^E A_e$.

Membership $m_{i,e}$ specifies how much sample $i$ belongs to environment $e$. With observed labels $g_i$, $m_{i,e}=\mathbf1[g_i=e]$. The optional inferred case uses $m_{i,e}=q_\eta(e\mid a_i)$ from auxiliary input $a_i$, with $\sum_e q_\eta(e\mid a_i)=1$. The sample adversary outputs scores $s_\theta(x_i)$ and defines
\begin{equation}
\pi_\theta(i)=\frac{\exp s_\theta(x_i)}{\sum_j\exp s_\theta(x_j)},\quad
M_e=\sum_i\pi_\theta(i)m_{i,e},\quad
\pi_\theta(i\mid e)=\frac{\pi_\theta(i)m_{i,e}}{M_e}.
\label{eq:weights}
\end{equation}
For $M_e>0$, conditional normalization changes emphasis within each environment while preserving the outer coefficients.

For hard memberships, conditioning cancels the global softmax denominator, so weights inside an environment depend on its samples' relative scores. Adding the same constant to all scores in one environment preserves its conditional risk. With soft memberships, samples can contribute to several environments, and the conditional distributions are linked through the shared scores and memberships. Both hard and soft memberships enter the same normalization in Eq.~\eqref{eq:weights}.

\subsection{Shared-risk coupling}
Define the conditional weighted risk and its TV-$\ell_1$ probe by
\begin{align}
R_{e,\theta}(w)&=\sum_i\pi_\theta(i\mid e)\ell_i(w),
&\Rmain&=\mathbb E_e[R_{e,\theta}(1)], \label{eq:risk}\\
a_e&=\left.\nabla_w R_{e,\theta}(w)\right|_{w=1},
&\PTV&=\mathbb E_e[|a_e|]. \label{eq:tv}
\end{align}
For scores and memberships independent of $w$, $a_e$ is the same weighted average of sample probes as used to construct prediction risk. This aligns the two components: high-loss emphasis within environments participates in the cross-environment stationarity objective.

Increasing a sample's conditional weight changes both prediction risk and its stationarity probe. Because TV is computed from the learnable weighted risk, its derivative contributes directly to the sample-adversary update. The adversary uses prediction and stationarity signals jointly with KL regularization, while the predictor minimizes their shared-risk objective.

Let $u_e(i)=m_{i,e}/\sum_jm_{j,e}$ and
\begin{equation}
\KLenv=\mathbb E_e\!\left[\KL\big(\pi_\theta(\cdot\mid e)\,\|\,u_e\big)\right].
\label{eq:kl}
\end{equation}
For hard membership, $u_e$ is uniform inside environment $e$; for soft membership it is membership-proportional. The observed-environment objective is
\begin{equation}
\min_\Phi\max_{\theta,\Psi}\ 
\mathcal J=\Rmain+\lambda(\Psi,\Phi)\PTV-\beta\KLenv,\qquad
\beta\geq0,\quad\lambda(\Psi,\Phi)\geq0.
\label{eq:objective}
\end{equation}
The negative sign makes KL oppose concentration during adversarial maximization. Uniform outer weighting assigns equal coefficients to the environment-level risks.

For comparison, let $\bar R_e$ be the uniform risk. Main-only replaces $\PTV$ in Eq.~\eqref{eq:objective} with $\mathbb E_e[|\nabla_w\bar R_e|_{w=1}]$ while retaining $R_{e,\theta}$ in prediction and the negative KL term. Global-tail changes conditional normalization; No-KL removes the regularizer. Appendix~\ref{app:ablation} specifies these controls.

\subsection{Block optimization}
The predictor $\Phi$ descends on $\mathcal J$, while $\theta$ and the TV multiplier block $\Psi$ ascend. An implementation-side annealing factor $\gamma_t$ gives $\lambda_{\mathrm{eff},t}=\gamma_t\lambda(\Psi,\Phi)$; $\gamma_{\max}$ controls its scale, whereas $\eta_\Psi$ is the multiplier learning rate. Appendix~\ref{app:algorithm} gives consistent block directions. At deployment, only $f_\Phi$ is retained.

With inferred memberships, the outer $\Phi,\theta,\Psi$ objective has the same form, conditional on $\eta$. Memberships are detached inside KL, and the inference block separately maximizes $\PTV$ only. This sequential block game uses auxiliary information to construct the memberships, while the predictor and sample adversary optimize the shared-risk objective conditional on those memberships.

\section{Theoretical Results}
\label{sec:analysis}
We analyze how conditional normalization and KL regularization structure the shared risk, then characterize the optimization problem and a specified update schedule. Appendix~\ref{app:theory} contains the complete proofs.

\subsection{Shared-risk properties}
Fix the memberships and suppose each environment has positive mass. Conditional weights are nonnegative, sum to one within each environment, and are independent of the dummy scalar $w$. Differentiating the weighted risk therefore gives
\[
a_e=\sum_i\pi_\theta(i\mid e)
\left.\nabla_w\ell_i(w)\right|_{w=1}.
\]
Prediction and stationarity use the same environment-conditioned coefficients, preserving a common sample-level measure for both objectives.

\begin{proposition}[Bounded conditional risks and probes]
\label{prop:bounded}
If $0\leq\ell_i(1)\leq L$ and $\left|\nabla_w\ell_i(w)\right|_{w=1}\leq G$, exact conditional normalization with positive environment mass gives
\[
0\leq R_{e,\theta}(1)\leq L,\qquad |a_e|\leq G,\qquad 0\leq\PTV\leq G.
\]
\end{proposition}
Conditional normalization preserves the convex-combination structure of both the risk and its probe. Thus, adversarial weighting can redistribute emphasis inside an environment while preserving these bounds and the prescribed outer coefficients. The proof is given in Appendix~\ref{app:bounded-proof}.

\subsection{KL-regularized tail selection}
The fixed-loss, risk-only subproblem has the following Gibbs solution.
Let $\mathcal I_e=\{i:u_e(i)>0\}$, and let $\Delta_e$ denote the probability simplex on $\mathcal I_e$.
\begin{proposition}[Entropy-smoothed within-environment risk]
\label{prop:gibbs}
For fixed losses, reference distribution $u_e$ on the environment support $\mathcal I_e$, and $\beta>0$,
\[
\max_{\rho\in\Delta_e}
\left\{\sum_i\rho_i\ell_i-\beta\KL(\rho\|u_e)\right\}
=\beta\log\sum_i u_e(i)e^{\ell_i/\beta},
\]
with optimizer
\begin{equation}
\rho_i^*=\frac{u_e(i)e^{\ell_i/\beta}}{\sum_j u_e(j)e^{\ell_j/\beta}}.
\label{eq:gibbs}
\end{equation}
\end{proposition}
The reference $u_e$ is uniform for hard memberships and membership-proportional for soft memberships. KL controls the exponential tilt toward high losses: as $\beta\to\infty$ the optimizer approaches $u_e$, while $\beta\to0$ concentrates mass on maximum-loss samples. Appendix~\ref{app:gibbs-proof} derives the smooth tail-risk value.

\paragraph{Stationarity-aware tail selection.}
For a single environment, consider the fixed-predictor, fixed-multiplier distributional relaxation. Let $b_i=\left.\nabla_w\ell_i(w)\right|_{w=1}$ denote the sample stationarity probe. Here $v$ is an analytic branch index for the absolute-value TV term, not a training block, update, or hyperparameter:
\[
\rho_i^\star(v)\propto u_e(i)\exp\!\left(\frac{\ell_i+\lambda v b_i}{\beta}\right),\qquad v\in\{-1,+1\}.
\]
Maximizing over the two sign branches gives the optimal value of this relaxation. Each branch tilts the reference by a joint loss--stationarity score, with $\beta$ controlling concentration. TV thus helps determine the weights that also define the stationarity-test risk. Appendix~\ref{app:gibbs-proof} derives this free-distribution representation; Appendix~\ref{app:algorithm} expands ECTR's parameterized update, whose current-point subgradient uses $s_e=\sign_0(a_e)$, distinct from the analytic index $v$.

\subsection{Optimization properties}
\paragraph{Compactness gives an optimized dual response.}
For known environments, define $V(\Phi)=\max_{\theta,\Psi}\mathcal J(\Phi,\theta,\Psi)$. Assumption~\ref{ass:regular} specifies compact feasible sets, continuity, positive environment mass, and conditions ensuring a bounded multiplier.
\begin{theorem}[Existence under compactness]
\label{thm:existence}
Under Assumption~\ref{ass:regular}, the known-environment problem $\min_\Phi\max_{\theta,\Psi}\mathcal J$ admits a solution, which is a semi-Nash solution in the sense of Definition~\ref{def:semi-nash}.
\end{theorem}
The optimized response value is continuous and attains its minimum on the compact predictor set; Appendix~\ref{app:existence-proof} gives the proof.

\paragraph{Summable updates yield finite-length trajectories.}
\begin{theorem}[Summable normalized block updates]
\label{thm:finite}
For a parameter block $z$ with descent/ascent direction $d_z^{(k)}$, use the normalized schedule in Appendix~\ref{app:finite-proof} with $p>1$. Then $\sum_{k=1}^\infty\|z^{(k+1)}-z^{(k)}\|_2<\infty$ and the trajectory converges to a limit in Euclidean parameter space. If iterates remain in a closed feasible set, the limit lies in it.
\end{theorem}
The bound $\|\Delta z_k\|\leq k^{-p}$ makes the trajectory Cauchy. Appendix~\ref{app:finite-proof} gives the schedule, proof, and distance and iteration bounds.

\section{Experiments}
\label{sec:experiments}
\subsection{Evaluation protocol}
We distinguish \emph{Original evaluation} from \emph{Extended evaluation}. Original benchmark tables use three-run mean $\pm$ standard deviation, original sensitivity studies use five-seed means, and all results in the component ablation (Tables~\ref{tab:core} and \ref{tab:ectr_ablation_mixed_compact}) are means over three random seeds. Extended studies use five seeds, with uncertainty indicated in the corresponding tables. Appendix~\ref{app:protocols} specifies batch-specific units, datasets, information access, and selection protocols.

Controlled LOGIT-Z jointly varies spurious correlation and feature-noise variance to create mixed shift. Other protocols include temporal House Price, constructed CelebA/Adult groups, held-out PACS styles, and geographic Landcover. NICO uses a within-context split. Worst-group and worst-environment metrics use each study's evaluation partition. Inferred-environment methods construct memberships from auxiliary variables.

\subsection{Controlled mixed-shift evaluation}
\paragraph{LOGIT-Z combines correlation reversal with increased variation.}
The synthetic construction holds invariant-feature agreement at $p_v=0.8$ while changing spurious-label agreement $p_s$ and feature-noise variance $\sigma^2$. Training environments have strong spurious alignment, whereas test environments reverse that relationship. Simultaneously, larger test feature-noise variance increases variation relative to training. Table~\ref{tab:app-controlled-mixed-settings} gives the three settings. For example, Hard changes spurious agreement from $(0.999,0.90)$ to $(0.20,0.001)$ and variance from $(0.10,0.20)$ to $(2.00,2.50)$. The resulting test challenges reliance on the spurious cue together with robustness to within-environment variation.

\begin{table}[htbp]
  \caption{Original evaluation: Controlled \texttt{LOGIT-Z} mixed-shift settings. We keep \(p_v=0.8\) and vary \(p_s\) and \(\sigma^2\) to control the strength of correlation reversal and diversity shift.}
  \label{tab:app-controlled-mixed-settings}
  \centering
  \small
  \setlength{\tabcolsep}{4.5pt}
  \begin{tabular}{lccccc}
    \toprule
    Setting & \(p_v\) & \(p_s^{\rm train}\) & \(p_s^{\rm test}\) & \(\sigma_{\rm train}^2\) & \(\sigma_{\rm test}^2\) \\
    \midrule
    Easy   & 0.8 & \(0.95, 0.85\)  & \(0.30, 0.10\)   & \(0.10, 0.20\) & \(0.80, 1.00\) \\
    Medium & 0.8 & \(0.98, 0.90\)  & \(0.20, 0.01\)   & \(0.10, 0.20\) & \(1.50, 2.00\) \\
    Hard   & 0.8 & \(0.999, 0.90\) & \(0.20, 0.001\)  & \(0.10, 0.20\) & \(2.00, 2.50\) \\
    \bottomrule
  \end{tabular}
\end{table}

Spurious-conflicting samples form a minority in training, alongside a spurious-aligned majority. This construction tests how tail emphasis helps a stationarity-based learner attend to difficult training examples. All methods in the controlled ablation use the same observed environment memberships.

\paragraph{Coupling changes the risk tested for stationarity.}

\begin{table}[htbp]
\centering
\caption{Original evaluation: controlled mixed-shift ablation. All results are means over three random seeds. Accuracy is in percent. Main-only evaluates TV on uniform risks; Full ECTR evaluates TV on the shared weighted risk.}
\label{tab:core}
\footnotesize\setlength{\tabcolsep}{4pt}
\begin{tabular}{llrrrr}
\toprule
Setting & Method & Worst Acc $\uparrow$ & Tail-TV & ESS ratio & Top-5 mass\\
\midrule
Hard & Global tail & 64.06 & 0.000592 & 0.689 & 0.126\\
& Main-only & 60.52 & 0.075835 & 0.079 & 0.683\\
& No KL & 62.95 & 0.000625 & 0.408 & 0.283\\
& Full ECTR & 68.44 & 0.000505 & 0.699 & 0.121\\
\midrule
Medium & Global tail & 68.29 & 0.000586 & 0.703 & 0.120\\
& Main-only & 65.00 & 0.065626 & 0.085 & 0.659\\
& No KL & 68.40 & 0.000594 & 0.519 & 0.198\\
& Full ECTR & 70.48 & 0.000509 & 0.702 & 0.120\\
\midrule
Easy & Global tail & 72.47 & 0.000626 & 0.710 & 0.117\\
& Main-only & 70.32 & 0.041404 & 0.111 & 0.575\\
& No KL & 72.19 & 0.000622 & 0.588 & 0.186\\
& Full ECTR & 73.56 & 0.000530 & 0.714 & 0.118\\
\bottomrule
\end{tabular}
\end{table}

Table~\ref{tab:core} shows that Full improves mean worst-group accuracy over Main-only by $7.92$, $5.48$, and $3.24$ percentage points in Hard, Medium, and Easy settings. This control tests whether stationarity participates in learning within-environment weights as well as using them, as expressed by the same-parameter-point adversary-direction difference in Appendix~\ref{app:algorithm}. Full also has higher mean ESS and lower mean Top-5 mass than Main-only, indicating more dispersed learned weights. The largest accuracy difference occurs in Hard, supporting shared-risk coupling under the strongest mixed shift in this experiment.

Full improves over Global-tail by $4.38$, $2.19$, and $1.09$ points. Global-tail uses global normalization and a global KL reference, whereas ECTR conditions both weighting and regularization on environments. No-KL has lower ESS and greater Top-5 mass than Full. Together, these comparisons support environment conditioning and concentration control in the coupled objective. The weight distributions in Appendix~\ref{app:diagnostics} show a heavier right tail for a subset of spurious-conflicting samples, consistent with selective emphasis within environments.

\subsection{Tail weighting under corruption}

\begin{table}[htbp]
\centering
\caption{Extended evaluation: selected training-corruption comparisons, five-seed mean $\pm$ sample standard deviation. Uniform is Uniform TV-IRM. House uses worst-environment MSE ($\downarrow$); CelebA uses worst-environment and Adult worst-group accuracy ($\uparrow$, 0--1). Complete results are in Table~\ref{tab:extended-noise}.}
\label{tab:noise-main}
\footnotesize\setlength{\tabcolsep}{3pt}
\begin{adjustbox}{max width=\linewidth}
\begin{tabular}{llcccc}
\toprule
Dataset & Noise & Uniform & Top-$\alpha$ & CVaR & ECTR\\
\midrule
House & .2 & $0.7004\pm0.0530$ & $0.6953\pm0.0445$ & $0.6841\pm0.0488$ & $0.6360\pm0.1567$\\
House & .4 & $0.6133\pm0.1097$ & $0.6135\pm0.1321$ & $0.6254\pm0.1352$ & $0.6428\pm0.0949$\\
CelebA & .1 & $0.5031\pm0.0123$ & $0.4445\pm0.0361$ & $0.4617\pm0.0216$ & $0.5147\pm0.0200$\\
CelebA & .2 & $0.4952\pm0.0066$ & $0.5010\pm0.0416$ & $0.4912\pm0.0327$ & $0.4987\pm0.0227$\\
Adult & .2 & $0.8058\pm0.0164$ & $0.7826\pm0.0236$ & $0.7785\pm0.0241$ & $0.8051\pm0.0246$\\
\bottomrule
\end{tabular}
\end{adjustbox}
\end{table}

The extended corruption study compares Uniform TV-IRM, environment-wise Top-$\alpha$, CVaR, and ECTR with shared data, architecture, predictor/multiplier initialization, training budget, corruption masks, and checkpoint rule. All four apply their weights to prediction and TV; the fixed-tail controls use loss-ranking rules, whereas ECTR learns weights through the coupled predictive, stationarity, and KL terms. On LOGIT-Z, CelebA, and Adult, fixed-tail fractions approximately match clean-training ECTR ESS and are then frozen across seeds and corruption levels. House uses a fixed tail fraction. Appendix~\ref{app:noise} defines these weighting rules.

Training labels or targets are corrupted while validation and test remain clean. The noise mass ratio (NMR) measures weight on corrupted samples relative to their prevalence. Across the evaluated corruption conditions, ECTR has lower mean NMR than both fixed-tail rules. At CelebA noise $.1$, its NMR is $0.9932\pm0.0164$, versus $2.0930\pm0.2364$ for Top-$\alpha$ and $2.0281\pm0.2327$ for CVaR. Across the prespecified conditions up to $.2$, ECTR has lower corruption mass than both fixed tails in all five paired seeds. NMR near one corresponds to aggregate weight proportional to the prevalence of corrupted examples.

Table~\ref{tab:noise-main} connects these weights to predictive performance. ECTR has the lowest mean House worst-environment MSE at noise $.2$. At CelebA $.1$, it improves over both fixed-tail rules in all five paired seeds. These findings show that learned tail weighting can retain predictive emphasis while controlling excess mass on corrupted training samples. Appendix~\ref{app:noise} provides the complete noise-rate comparisons, predictive results, and weighting diagnostics.

\subsection{Benchmark comparisons}
\paragraph{Original evaluation compares matched information settings.}
Table~\ref{tab:parent-main} compares ECTR with its closest TV-based parent methods. The observed-environment panel includes IRM-TV and OOD-TV-IRM; the inferred-environment panel includes Minimax-TV and OOD-TV-Minimax. These are three-run results in the Original evaluation. House tests temporal transfer, CelebA uses a constructed target--attribute correlation, and Adult evaluates demographic groups.

\begin{table}[htbp]
\centering
\caption{Original evaluation: comparisons with TV-based parent methods, mean $\pm$ standard deviation over three runs. House uses MSE reported to four decimal places ($\downarrow$); CelebA and Adult use accuracy in percent ($\uparrow$). Observed and inferred environments are separate information settings. A dash indicates an unavailable entry. Bold denotes the best mean within each panel.}
\label{tab:parent-main}
\footnotesize\setlength{\tabcolsep}{3pt}
\begin{adjustbox}{max width=\linewidth}
\begin{tabular}{lrrrrrr}
\toprule
& \multicolumn{2}{c}{House MSE $\downarrow$} & \multicolumn{2}{c}{CelebA Acc. $\uparrow$} & \multicolumn{2}{c}{Adult Acc. $\uparrow$}\\
Method & Mean & Worst & Mean & Worst & Mean & Worst\\
\midrule
\multicolumn{7}{l}{\textit{Observed environments}}\\
IRM-TV-$\ell_1$ & 0.4615$\pm$0.0620 & 0.6306$\pm$0.0852 & - & - & 82.31$\pm$0.36 & 79.80$\pm$0.96\\
OOD-TV-IRM-$\ell_1$ & 0.4397$\pm$0.0681 & 0.5939$\pm$0.0605 & \textbf{76.88$\pm$0.63} & 72.81$\pm$0.75 & 83.05$\pm$0.63 & 81.08$\pm$1.09\\
ECTR & \textbf{0.3958$\pm$0.0820} & \textbf{0.5461$\pm$0.0769} & 76.82$\pm$1.27 & \textbf{73.58$\pm$1.73} & \textbf{83.45$\pm$0.64} & \textbf{81.31$\pm$0.92}\\
\midrule
\multicolumn{7}{l}{\textit{Inferred environments}}\\
Minimax-TV-$\ell_1$ & 0.3695$\pm$0.0658 & 0.5161$\pm$0.0835 & - & - & 82.32$\pm$0.78 & 79.85$\pm$0.96\\
OOD-TV-Minimax-$\ell_1$ & 0.3628$\pm$0.0654 & 0.4891$\pm$0.0763 & 74.42$\pm$1.57 & 66.38$\pm$1.31 & 82.53$\pm$0.99 & 80.24$\pm$1.34\\
ECTR & \textbf{0.3452$\pm$0.0589} & \textbf{0.4574$\pm$0.0573} & \textbf{75.24$\pm$1.35} & \textbf{71.22$\pm$1.28} & \textbf{83.56$\pm$0.57} & \textbf{81.15$\pm$0.63}\\
\bottomrule
\end{tabular}
\end{adjustbox}
\end{table}

With observed environments, ECTR reduces House mean MSE to $0.3958$ and worst MSE to $0.5461$, compared with $0.4397$ and $0.5939$ for OOD-TV-IRM. On CelebA, ECTR improves worst-group accuracy to $73.58$ from $72.81$; on Adult, it reaches $81.31$ compared with $81.08$. With inferred environments, ECTR has lower House MSE and higher CelebA and Adult accuracy than the listed TV-based parents. Its worst House MSE is $0.4574$ versus $0.4891$ for OOD-TV-Minimax, while CelebA worst accuracy is $71.22$ versus $66.38$. These comparisons show the benefit of conditional tail weighting in both information settings.

The complete Original evaluation in Appendix~\ref{app:original} includes broader baselines, Landcover, and vision tasks. On Landcover, ECTR reaches $63.80$ OOD accuracy compared with $63.72$ for OOD-TV-Minimax. PACS evaluates held-out styles, and NICO evaluates prediction under context diversity.

\paragraph{Extended evaluation compares additional invariant-learning baselines.}

\begin{table}[htbp]
\centering
\caption{Extended evaluation: additional invariant-learning comparisons over five prespecified seeds, with means and accompanying uncertainty. House Price uses MSE; CelebA accuracy is in percent. WRI uses frozen-representation classification; a dash marks an unavailable regression entry. Full results appear in Table~\ref{tab:extended-baselines}.}
\label{tab:performance-main}
\footnotesize\setlength{\tabcolsep}{3pt}
\begin{adjustbox}{max width=\linewidth}
\begin{tabular}{llcccc}
\toprule
Dataset & Metric & V-REx & RDM & WRI & ECTR\\
\midrule
House Price & Test MSE $\downarrow$ & $0.5567\pm0.0603$ & $0.6434\pm0.0947$ & - & $0.4152\pm0.0920$\\
& Worst-env MSE $\downarrow$ & $0.7316\pm0.0906$ & $0.8060\pm0.1500$ & - & $0.5986\pm0.1182$\\
CelebA & Test Acc. $\uparrow$ & $70.80\pm0.24$ & $69.82\pm0.38$ & $70.07\pm0.63$ & $77.56\pm1.83$\\
& Worst-env Acc. $\uparrow$ & $60.08\pm0.34$ & $57.89\pm0.75$ & $58.41\pm1.25$ & $73.93\pm2.42$\\
\bottomrule
\end{tabular}
\end{adjustbox}
\end{table}

The extended baseline study compares V-REx, RDM, WRI, and ECTR using the same five prespecified seeds, splits, observed environments, and clean-validation protocol. Hyperparameters are selected on seed 1 validation and fixed for the five formal seeds. Table~\ref{tab:performance-main} shows the House and CelebA comparisons. Within this batch, ECTR has the best mean on all eight metrics in the complete four-dataset table, with higher accuracy and lower MSE. On LOGIT-Z, worst-environment accuracy is $42.25\pm22.05$ for ECTR and $39.87\pm23.85$ for RDM. Appendix~\ref{app:baselines} also includes the Colored MNIST results.

The extended LOGIT-Z comparison uses the Hard configuration in Table~\ref{tab:app-controlled-mixed-settings}, combining strong spurious-correlation reversal with increased feature-noise variance. Seeds vary the data realization, split, and training randomness, so dispersion describes run-to-run variation. Within matched runs, ECTR improves test accuracy in all five pairs against both V-REx and RDM, with mean gains of $2.28$ and $2.30$ percentage points; Appendix~\ref{app:baselines} reports the paired summaries.

\subsection{Shared-membership comparisons}
Two extended protocols examine environment information: end-to-end auxiliary degradation and identical frozen memberships for TV-IRM and ECTR (Appendix~\ref{app:information}). House uses construction year, a proxy for decade-based groups. With full auxiliary information, ECTR's overall MSE is $0.4078$; with shuffled input, it is $0.5120$. Under shared fixed memberships, all listed mean comparisons favor ECTR, and moderately softened House memberships improve on the hard partition. The shared-membership comparison supports the benefit of the coupled objective under matched environment assignments.

On Adult, ECTR's overall and worst-group accuracy remain nearly unchanged as membership NMI changes from $0.0068$ to $0.0002$ under complete auxiliary row corruption. The fixed-membership study likewise finds similar accuracy across oracle, permuted, and uniform assignments. Together, the two datasets illustrate how the coupled objective operates with different grouping information: House benefits from informative memberships, while Adult maintains predictive performance across the tested assignments.

\subsection{Concentration, aggregation, and cost}
The expanded Medium LOGIT-Z grid contains $63$ healthy configurations, $15$ within $0.01$ worst-validation accuracy of the best healthy setting. Healthy means no seed collapses, mean minimum-environment ESS at least $0.20$, and mean Top-5 mass below $0.50$. In the Hard stress test, $\eta_\theta=10^{-2}$ and $\beta=0.05$ give ESS $0.0009$ and collapse in $5/5$ runs; $\beta=10$ gives ESS $0.9479$ and $0/5$ collapse. Stable settings also occur at $\beta=0$. These results show how KL and adversary strength jointly regulate concentration.

Across the effective-TV sweep $\gamma_{\max}=1$--$100$, LOGIT-Z worst-test accuracy varies by $0.12$ percentage points. House worst-validation MSE ranges from $0.3534$ to $0.3543$, and worst-test MSE ranges from $0.6114$ to $0.7515$. Appendix~\ref{app:stability} provides the corresponding ESS and collapse statistics.

Changing outer aggregation to size-proportional or risk-adaptive weighting yields at most $0.104$ percentage-point mean paired difference in the tested clean LOGIT-Z settings. Uniform averaging provides a simple default in this comparison, and Appendix~\ref{app:aggregation} reports the associated environment weights and corruption contributions.

On the full-batch House predictor, ECTR/TV-IRM training-time ratios are $1.340$ with known environments and $1.257$ with inferred environments; peak-memory ratios are $1.034$ and $1.033$. Both deploy the same $545$-parameter predictor with approximately $0.049$ ms batch-1 latency. ECTR's additional weighting machinery is used during training, preserving the deployment architecture (Appendix~\ref{app:cost}).

\section{Conclusion}
ECTR addresses mixed distribution shifts by coupling cross-environment TV invariance with within-environment adversarial tail weighting. Its central contribution is a shared environment-conditioned risk: the same sample weights shape prediction and the stationarity probe. Conditional normalization preserves the prescribed outer environment weighting, while KL regularization controls concentration within environments.

This formulation admits a unified interpretation as a controlled change of empirical measure. Our analysis connects bounded conditional risks and probes to a stationarity-dependent weighting signal and a KL-regularized distributional interpretation: TV participates in learning the weights that also determine the stationarity test. The known-environment minimax problem admits a solution under the stated compactness and continuity conditions, and the specified summable normalized-update schedule yields finite-length, convergent parameter trajectories.

Controlled mixed-shift ablations support the shared-risk design. Across synthetic and real-world benchmarks, ECTR achieves favorable performance under both observed- and inferred-environment settings, with improvements over its TV-based parent methods and strong results against additional invariant-learning baselines. Corruption studies further demonstrate reduced excess weight on corrupted samples relative to the tested fixed-tail rules, while shared-membership comparisons support the coupled objective under matched assignments. Future work will explore noise-aware tail modeling to better distinguish informative hard examples from corrupted high-loss samples, together with richer tail structures and broader environment coverage.

\label{maintext:end}
\subsection*{AI use statement}

In this work, we used generative AI tools to assist with LaTeX editing, language polishing, manuscript organization and formatting, and checking the presentation of manuscript revisions. We did not use generative AI tools to generate or modify experimental data, select experimental outcomes, or determine the scientific conclusions of this work. Other disclosure tasks are not applicable to this work. We have reviewed all AI-assisted work, including checking edited text, mathematical expressions, references, and LaTeX against the authors' source materials and experimental records. We take responsibility for the final content of this work, including text, claims or artifacts produced with the aid of generative AI.

\subsection*{Reproducibility statement}
Appendix~\ref{app:theory} provides the assumptions, proofs, and block-update algorithm. Appendix~\ref{app:protocols} describes the data and evaluation protocols, and Appendices~\ref{app:original}--\ref{app:diagnostics} provide complete results and diagnostics. Original and extended evaluations are presented separately with their seed and uncertainty conventions. We will publicly release the source code upon acceptance of this paper.

\subsection*{Ethics statement}
The experiments use existing benchmark data, including demographic and face-attribute groups, to evaluate predictive performance. Group-level results describe accuracy under the specified evaluation protocols.

\bibliography{main}
\bibliographystyle{iclr2027_conference}
\clearpage
\appendix
\section{Mathematical details}
\label{app:theory}
We analyze the ECTR objective using exact normalized conditional weights and the L1 stationarity penalty in Eq.~\eqref{eq:tv}. The analysis establishes bounded risks and probes, an entropy-regularized risk-only solution, compact minimax existence, and finite-length trajectories under a specified update schedule.

\subsection{Proof of Proposition~\ref{prop:bounded}}
\label{app:bounded-proof}
\begin{assumption}[Regularity and boundedness]
\label{ass:regular}
The feasible sets of the relevant parameter blocks are nonempty, bounded, and closed. Losses, tail scores, and the nonnegative multiplier are continuous in their parameters, and the loss is differentiable in the dummy scalar $w$. The stationarity probe is continuous in the variables on which it depends. Each environment has mass $M_e\geq\varepsilon_m>0$. For the following boundedness statement, per-sample losses and dummy-scalar gradients are bounded, and the weighting scores and memberships do not depend on $w$.
\end{assumption}
Compact parameter sets and a continuous multiplier provide its boundedness. A softplus or exponential parameterization supplies nonnegativity; a capped parameterization or an explicitly bounded feasible set supplies the upper-bound condition used in the analysis.

\begin{proof}
For every $e$, $\pi_\theta(i\mid e)\geq0$ and $\sum_i\pi_\theta(i\mid e)=1$, so
$R_{e,\theta}=\sum_i\pi_\theta(i\mid e)\ell_i$ is a convex combination in $[0,L]$.
Independence from $w$ gives
\[
a_e=\sum_i\pi_\theta(i\mid e)\left.\nabla_w\ell_i(w)\right|_{w=1},\qquad
|a_e|\leq\sum_i\pi_\theta(i\mid e)G=G.
\]
Averaging absolute probes yields the last bound.
\end{proof}
The squared-probe variant $\mathbb E_e[a_e^2]$ is bounded by $G^2$ under the same assumptions. ECTR uses the L1 probe in Eq.~\eqref{eq:tv}, whose bound is $G$.

The convex-combination argument uses exact normalization by the positive environment mass $M_e$ in Eq.~\eqref{eq:weights}.

\subsection{Proof of Proposition~\ref{prop:gibbs}}
\label{app:gibbs-proof}
\begin{proof}
The Lagrangian with normalization multiplier $\mu$ is
\[
\mathcal A(\rho,\mu)=\sum_i\rho_i\ell_i-\beta\sum_i\rho_i\log(\rho_i/u_e(i))
+\mu\left(\sum_i\rho_i-1\right).
\]
Differentiation gives $\ell_i-\beta(\log(\rho_i/u_e(i))+1)+\mu=0$.
Thus $\rho_i=u_e(i)\exp((\ell_i+\mu-\beta)/\beta)$; normalization yields the stated distribution. Substituting this distribution gives the log-sum-exp value.
\end{proof}
In this fixed-loss risk-only subproblem, the optimizer approaches $u_e$ as $\beta\to\infty$ and concentrates on maximum-loss samples as $\beta\to0$. This interpolation explains how KL controls tail emphasis. The ECTR neural adversary combines the predictive and TV terms and shares parameters across samples, with conditional distributions defined by Eq.~\eqref{eq:weights}.

\paragraph{Stationarity-aware distributional relaxation.}
For the fixed-predictor, fixed-multiplier distributional relaxation, fix known environment memberships, $\Phi$, and $\Psi$, so that $\lambda=\lambda(\Psi,\Phi)$ is finite, nonnegative, and independent of $\theta$. Write $\ell_i=\ell_i(1)$ and $b_i=\left.\nabla_w\ell_i(w)\right|_{w=1}$. On the reference support $\mathcal I_e$, assume $u_e(i)>0$ and finite $\ell_i,b_i$. For $\beta>0$, consider the free within-environment distribution $\rho\in\Delta_e$, the probability simplex on $\mathcal I_e$:
\[
F_e(\rho)=\sum_{i\in\mathcal I_e}\rho_i\ell_i
+\lambda\left|\sum_{i\in\mathcal I_e}\rho_i b_i\right|
-\beta\KL(\rho\|u_e).
\]
Since $|a|=\max_{v\in\{-1,+1\}}va$ and both optimizations are maximizations,
\[
\begin{aligned}
\max_{\rho\in\Delta_e}F_e(\rho)
&=\max_{v\in\{-1,+1\}}\max_{\rho\in\Delta_e}
\left\{\sum_{i\in\mathcal I_e}\rho_i(\ell_i+\lambda v b_i)-\beta\KL(\rho\|u_e)\right\}\\
&=\max_{v\in\{-1,+1\}}\beta\log\sum_{i\in\mathcal I_e}u_e(i)
\exp\!\left(\frac{\ell_i+\lambda v b_i}{\beta}\right).
\end{aligned}
\]
The second equality applies Proposition~\ref{prop:gibbs} to each fixed branch score $\ell_i+\lambda v b_i$; these scores need only be finite, not nonnegative. For each sign branch $v$, the maximizing distribution is
\[
\rho_i^\star(v)=\frac{u_e(i)\exp((\ell_i+\lambda v b_i)/\beta)}
{\sum_{j\in\mathcal I_e}u_e(j)\exp((\ell_j+\lambda v b_j)/\beta)}.
\]
Choose a branch attaining the larger log-sum-exp value. Its distribution attains that branch value and hence the maximum of $F_e$. If the branch values tie, either branch optimizer is optimal, allowing multiple optimal choices. The signed probe $b_i$ modifies the loss score through $\lambda v b_i$, and KL controls concentration around the reference measure. In Main-only, the uniform-risk TV term is constant in $\rho$ at fixed $\Phi,\Psi$, so its distributional optimizer retains the loss-only tilt of Proposition~\ref{prop:gibbs}.

This relaxation optimizes free distributions separately within environments. ECTR instead generates conditional weights through the shared parameter $\theta$ in Eq.~\eqref{eq:weights}; Appendix~\ref{app:algorithm} expands its parameterized update. At a fixed annealing time, the same derivation uses $\gamma_t\lambda$ in place of $\lambda$.

\subsection{Minimax existence}
\label{app:existence-proof}
\begin{definition}[Semi-Nash solution]
\label{def:semi-nash}
Let $V(\Phi)=\max_{\theta,\Psi}\mathcal J(\Phi,\theta,\Psi)$ for observed environments. A triple $(\Phi^*,\theta^*,\Psi^*)$ is a semi-Nash solution if $(\theta^*,\Psi^*)$ maximizes $\mathcal J$ at $\Phi^*$ and $\Phi^*$ minimizes $V$.
\end{definition}
\begin{proof}[Proof of Theorem~\ref{thm:existence}]
Positive mass makes the conditional weights continuous. The risks, KL term, and stationarity probe are continuous; absolute value preserves continuity, as does multiplication by the continuous multiplier. Hence $\mathcal J$ is continuous in all blocks. For fixed $\Phi$, the compact dual feasible set and the extreme value theorem give a maximizer $(\theta_{\max}(\Phi),\Psi_{\max}(\Phi))$. The maximum theorem makes $V$ continuous on the compact primal feasible set, so it attains its minimum at some $\Phi^*$. Choose a dual maximizer at $\Phi^*$. Then
\[
\mathcal J(\Phi^*,\theta^*,\Psi^*)\geq\mathcal J(\Phi^*,\theta,\Psi)
\quad\text{for all }\theta,\Psi,
\]
and $\mathcal J(\Phi^*,\theta^*,\Psi^*)=V(\Phi^*)\leq V(\Phi)$ for every $\Phi$.
\end{proof}
The theorem characterizes minimization of $V(\Phi)$ against the optimized dual response for known environments. The optional inferred-environment block uses the separate TV-only update described below.

\subsection{Coupled block updates}
\label{app:algorithm}
For the scalar probe define $\sign_0(a)=\sign(a)$ for $a\neq0$, and $\sign_0(0)=0$. A selected subgradient for a block $z$ is
\begin{equation}
\partial_z\PTV=\mathbb E_e[\sign_0(a_e)\partial_z a_e].
\label{eq:subgrad}
\end{equation}
For scores depending on $x_i$ and $\theta$, the known-environment directions are
\begin{align}
\partial_\Phi\mathcal J
 &=\partial_\Phi\Rmain+\lambda\partial_\Phi\PTV
   +\PTV\nabla_\Phi\lambda,\\
\partial_\theta\mathcal J
 &=\partial_\theta\Rmain+\lambda\partial_\theta\PTV-\beta\partial_\theta\KLenv,\\
\nabla_\Psi\mathcal J&=\PTV\nabla_\Psi\lambda .
\end{align}
If $\lambda$ is independent of $\Phi$, its $\Phi$ derivative vanishes. With annealing, replace $\lambda$ with $\gamma_t\lambda$. The predictor descends; $\theta$ and $\Psi$ ascend. Inferred membership gradients in KL are stopped, and the separate inference direction is $+\partial_\eta\PTV$ only.

\paragraph{Sample-level signals drive the weight update.}
Fix known environment memberships, $\Phi$, and $\Psi$. Then $\lambda=\lambda(\Psi,\Phi)$ is a finite nonnegative constant with respect to $\theta$. Let $\ell_i=\ell_i(1)$, $b_i=\left.\nabla_w\ell_i(w)\right|_{w=1}$, $p_{i\mid e}=\pi_\theta(i\mid e)$, and $s_e=\sign_0(a_e)$, where $a_e=\sum_{i\in\mathcal I_e}p_{i\mid e}b_i$. The fixed support $\mathcal I_e$ has $u_e(i)>0$, and the conditional weights are positive on that support and independent of the dummy scalar $w$. With $\beta\geq0$, differentiating the predictive and KL terms and taking the selected TV subgradient gives
\[
\begin{aligned}
\partial_\theta\mathcal J
&=\frac1E\sum_e\sum_{i\in\mathcal I_e}
\left[\ell_i+\lambda s_e b_i-\beta\left(\log\frac{p_{i\mid e}}{u_e(i)}+1\right)\right]
\nabla_\theta p_{i\mid e}\\
&=\frac1E\sum_e\sum_{i\in\mathcal I_e}
\left[\ell_i+\lambda s_e b_i-\beta\log\frac{p_{i\mid e}}{u_e(i)}\right]
\nabla_\theta p_{i\mid e}.
\end{aligned}
\]
The predictive risk contributes $\ell_i$; the weighted TV term contributes $\lambda s_e b_i$. The KL derivative contributes $\log(p_{i\mid e}/u_e(i))+1$, whose constant term cancels because $\sum_{i\in\mathcal I_e}\nabla_\theta p_{i\mid e}=\nabla_\theta 1=0$. At $a_e=0$, this expression uses the selected subgradient $\sign_0(0)=0$. The probe $b_i$ retains its sign: the TV signal depends on its contribution to the environment's aggregate stationarity probe.

At the same $(\Phi,\theta,\Psi)$, with identical conditional weights and KL terms, the Full and Main-only adversary directions differ by
\[
\partial_\theta\mathcal J_{\mathrm{Full}}-\partial_\theta\mathcal J_{\mathrm{Main}}
=\frac{\lambda}{E}\sum_e s_e\sum_{i\in\mathcal I_e}b_i\nabla_\theta p_{i\mid e}.
\]
Main-only's uniform-risk TV is constant in $\theta$, so it supplies no such weighting signal. This is a same-parameter-point comparison of the objectives, rather than a gradient comparison between separately trained models. With annealing, use the fixed-time coefficient $\gamma_t\lambda$. Thus, TV participates in learning the conditional sample distribution, and that same distribution determines the risk on which stationarity is evaluated.

\begin{algorithm}[tbp]
\caption{ECTR block-update schematic with L1 TV}
\label{alg:ectr}
\begin{algorithmic}[1]
\REQUIRE Batch, observed labels or auxiliary inputs, $\beta$, block step sizes, fixed dummy $w=1$.
\STATE Compute $\ell_i(w)$, sample scores, and observed or soft memberships.
\STATE Form $\pi_\theta(i\mid e)$, $R_{e,\theta}$ and $\Rmain$ using Eqs.~\eqref{eq:weights}--\eqref{eq:risk}.
\STATE Form $\PTV=\mathbb E_e[|\nabla_w R_{e,\theta}|_{w=1}]$.
\STATE Compute $\KLenv$ with memberships detached if inferred.
\STATE Form $\mathcal J_t=\Rmain+\gamma_t\lambda(\Psi,\Phi)\PTV-\beta\KLenv$.
\STATE Update $\theta$ by ascent on $\mathcal J_t$.
\STATE Recompute needed quantities; update $\Psi$ by ascent on $\gamma_t\lambda\PTV$.
\IF{environments are inferred}
\STATE Recompute needed quantities; update $\eta$ by ascent on $\PTV$ only.
\ENDIF
\STATE Recompute needed quantities; update $\Phi$ by descent on $\mathcal J_t$.
\STATE Keep only $f_\Phi$ for deployment.
\end{algorithmic}
\end{algorithm}
Algorithm~\ref{alg:ectr} summarizes the block objectives and ascent/descent directions. Different block orders or multiple ascent steps preserve these objective assignments.
\FloatBarrier

\subsection{Finite-length parameter trajectories}
\label{app:finite-proof}
The normalized schedule in Theorem~\ref{thm:finite} chooses, for a parameter block $z$ with descent/ascent direction $d_z^{(k)}$,
\[
\alpha_z^{(k)}=
\begin{cases}
1/(k^p\|d_z^{(k)}\|_2),&d_z^{(k)}\neq0,\\
0,&d_z^{(k)}=0,
\end{cases}
\qquad p>1.
\]
For the limit in Theorem~\ref{thm:finite}, the distance is less than $\epsilon$ whenever
\[
N>((p-1)\epsilon)^{-1/(p-1)}+1.
\]
\begin{proof}[Proof of Theorem~\ref{thm:finite}]
Every update has length at most $k^{-p}$, including the zero-direction case. The $p$-series converges, so the total path length is finite. For $m>0$,
\[
\|z^{(k+m)}-z^{(k)}\|_2
\leq\sum_{j=k}^{k+m-1}\|z^{(j+1)}-z^{(j)}\|_2
\leq\sum_{j=k}^{\infty}j^{-p}\longrightarrow0.
\]
The sequence is Cauchy and Euclidean space is complete. Closedness gives the feasible-limit statement. Finally, for its limit $z^\bullet$,
\[
\|z^{(N)}-z^\bullet\|_2
\leq\sum_{k=N}^\infty k^{-p}
<\int_{N-1}^\infty t^{-p}\,dt
=\frac{1}{(p-1)(N-1)^{p-1}},
\]
which yields the stated sufficient $N$.
\end{proof}
The Euclidean $\ell_2$ norm measures parameter-update length in this schedule, while the $\ell_1$ norm defines the TV penalty. The argument applies blockwise to $\Phi,\theta,\Psi$ and, when present, $\eta$. The distance bound quantifies convergence to each trajectory's own limit under the stated summable-length schedule.

\section{Data and evaluation protocols}
\label{app:protocols}
\subsection{Evaluation batches and information access}
Original benchmark results use three-run mean $\pm$ standard deviation. Original LOGIT-Z sensitivity studies use five-seed means. All results in the component ablation, including every accuracy and diagnostic entry in Tables~\ref{tab:core} and \ref{tab:ectr_ablation_mixed_compact}, are means over three random seeds. Extended studies use five seeds; the corruption tables use mean $\pm$ sample standard deviation, and tables of means are identified in their captions. Original and extended results are presented separately in their batch-specific units.

Known-environment ECTR receives observed training environment labels. Inferred-environment ECTR receives auxiliary inputs for $q_\eta$, with group/domain identities used to define evaluation, splits, or validation as specified by each dataset. Original comparison methods include IRM, groupDRO, IRM-TV-$\ell_1$, OOD-TV-IRM-$\ell_1$, EIIL, LfF, ZIN, TIVA, Minimax-TV-$\ell_1$, and OOD-TV-Minimax-$\ell_1$ \citep{arjovsky2019irm,sagawa2020groupdro,lai2024tvirm,wang2025oodtvirm,creager2021eiil,nam2020LFF,lin2022zin,tan2023TIVAtiva}.

For the extended baseline study, V-REx, RDM, WRI, and ECTR use the same splits, observed environments, clean validation, and five prespecified seeds. Hyperparameters are chosen on seed 1 validation and fixed across the five formal seeds. WRI uses frozen representations for classification. The aggregation study selects adaptive temperature on worst-environment validation using three seeds, then fixes it for five paired evaluations.

\subsection{Dataset protocols}
\paragraph{LOGIT-Z combines spurious alignment and feature variation.}
The synthetic protocol contains spurious-aligned majority group A and spurious-conflicting minority group B. It keeps invariant-feature agreement $p_v=0.8$, changes spurious-label agreement $p_s$ between training and test, and increases feature-noise variances to construct mixed shift. Table~\ref{tab:app-controlled-mixed-settings} lists all three settings. The original Hard sensitivity study uses $5000$ training and $5000$ test samples, with approximately $94.82\%$ majority and $5.18\%$ minority samples. Extended baseline LOGIT-Z uses mixed-hard; the expanded grid uses Medium with separate Hard stress tests; the effective-TV study uses correlation-only shift.

\paragraph{House Price evaluates temporal transfer.}
The regression protocol uses $15$ variables and follows the temporal construction of prior work \citep{lin2022zin,lai2024tvirm}. Training years are $[1900,1950]$ and testing years $(1950,2000]$. The Original protocol normalizes prices within built year and defines five training decade bins. Inferred environments use built year as auxiliary information. Mean and worst MSE are evaluated on test environments. The corruption study fixes normalization statistics computed from clean training targets.

\paragraph{CelebA uses a constructed target--attribute correlation.}
The target is Smiling, deliberately correlated with Gender \citep{liu2015celeba}. The protocol extracts $512$-dimensional features using a pretrained ResNet18 \citep{he2016resnet}. Gender provides the known environment indicator. Inference uses Young, Blond Hair, Eyeglasses, High Cheekbones, Big Nose, Bags Under Eyes, and Chubby. Extended corruption flips Smiling within environment, clean-label, and Male-attribute strata. Worst metrics use each study's evaluation partition.

\paragraph{Adult Income evaluates demographic groups.}
The target is income above \$50K/yr. Four evaluation groups combine race (Black/Non-Black) and sex (Male/Female). Training uses two-thirds of samples from Black Male and Non-Black Female, and evaluation covers all four groups. Auxiliary inputs are Age, FNLWGT, Education-Number, Capital-Gain, Capital-Loss, and Hours-Per-Week. Other categorical variables, excluding race and sex, are one-hot encoded and concatenated with numerical features for normalization; the preprocessing protocol includes optional PCA reduction of the encoded categorical features.

\paragraph{Landcover evaluates geographic transfer.}
This satellite time-series task uses $46\times8$ inputs, six classes, and a 1D-CNN \citep{gislason2006Landecoverrandomforest,xie2021Landcoverinnout}. Latitude/longitude provide auxiliary information for environment inference. Training is non-African; testing covers disjoint non-African data (IID) and Africa (OOD). The two result columns identify these respective test sets.

\paragraph{Colored MNIST uses spurious color cues.}
The protocol uses spurious color cues \citep{arjovsky2019irm}. The Original setting removes explicit background-color information and supplies mean RGB values as auxiliary input. Its table reports mean accuracy. The Extended table reports test and worst-environment accuracy separately.

\paragraph{PACS evaluates held-out styles.}
PACS has four styles (photo, art, cartoon, sketch), seven classes, and $9{,}991$ images of size $224\times224$ \citep{li2017PACS,gulrajani2021Domainbedlostdg}. Leave-one-domain-out evaluation trains on three styles and tests on the held-out style. Domain labels are provided in known mode and withheld from model input in inferred mode. They define training/validation splits for model selection. The table reports the three-run mean.

\paragraph{NICO evaluates prediction under context diversity.}
The NICO benchmark \citep{zhang2023nicoDG} includes Vehicle (nine classes) and Animal (ten classes), with samples split $80\%/20\%$ within each context. Context annotations are removed in the inferred setting, and mean RGB provides auxiliary information. Known- and inferred-environment results appear in separate rows.

\subsection{Prediction and weighting metrics}
Classification accuracy is higher-is-better; MSE is lower-is-better. Original Mean denotes the environment average where defined, and Extended Test/Overall denotes the test metric of that study. Worst is the minimum accuracy or maximum MSE over the specified evaluation groups/environments. Tail-TV, ESS, Top-5 mass, and Env Mass Skew are defined in Appendix~\ref{app:ablation}; NMR is defined in Appendix~\ref{app:noise}. Algorithm~\ref{alg:ectr} specifies the optimization blocks, while the sensitivity studies identify the varied learning rates and regularization coefficients.
\FloatBarrier

\clearpage
\section{Original evaluation}
\label{app:original}
\subsection{Observed and inferred environments}
Tables~\ref{tab:t1} and \ref{tab:t2} give the complete three-run Original evaluation. Table~\ref{tab:t1} separates observed-environment methods from methods evaluated without training environment labels, with House MSE to four decimal places and accuracy in percent. Bold identifies the best mean within each information setting.

\begin{table}[htbp]
\centering
\caption{Original evaluation: structured tasks, mean $\pm$ standard deviation over three runs. The first panel reports House Price and CelebA; the second reports Landcover and Adult. House Price uses MSE ($\downarrow$), reported to four decimal places; all other metrics are accuracy in percent ($\uparrow$). Landcover columns are IID/OOD. Bold denotes the best mean within each information setting.}
\label{tab:t1}
\footnotesize\setlength{\tabcolsep}{3pt}
\begin{adjustbox}{max width=\linewidth}
\begin{tabular}{lcccc}
\toprule
Method & House Mean & House Worst & CelebA Mean & CelebA Worst \\
\midrule
\multicolumn{5}{l}{\textit{Observed-environment variants}} \\
IRM~\citep{arjovsky2019irm} & 0.4834$\pm$0.0735 & 0.6569$\pm$0.0941 & 75.23$\pm$3.19 & 72.16$\pm$3.42 \\
IRM-TV-$\ell_1$~\citep{lai2024tvirm} & 0.4615$\pm$0.0620 & 0.6306$\pm$0.0852 & -- & -- \\
OOD-TV-IRM-$\ell_1$~\citep{wang2025oodtvirm} & 0.4397$\pm$0.0681 & 0.5939$\pm$0.0605 & \textbf{76.88$\pm$0.63} & 72.81$\pm$0.75 \\
ECTR (ours) & \textbf{0.3958$\pm$0.0820} & \textbf{0.5461$\pm$0.0769} & 76.82$\pm$1.27 & \textbf{73.58$\pm$1.73} \\
\midrule
\multicolumn{5}{l}{\textit{Methods evaluated without environment labels}} \\
ERM~\citep{vapnik1998ERM} & 0.4661$\pm$0.0059 & 0.6419$\pm$0.0184 & 67.58$\pm$6.81 & 47.12$\pm$7.25 \\
EIIL~\citep{creager2021eiil} & 0.8829$\pm$0.0298 & 0.9314$\pm$0.0307 & 69.54$\pm$4.27 & 49.45$\pm$6.73 \\
LfF~\citep{nam2020LFF} & -- & -- & 56.36$\pm$1.27 & 49.37$\pm$1.40 \\
ZIN~\citep{lin2022zin} & 0.3525$\pm$0.0568 & 0.5182$\pm$0.0930 & 74.26$\pm$0.48 & 65.67$\pm$0.94 \\
TIVA~\citep{tan2023TIVAtiva} & 0.3590$\pm$0.0266 & 0.4657$\pm$0.0389 & 69.55$\pm$1.42 & 59.42$\pm$1.12 \\
Minimax-TV-$\ell_1$~\citep{lai2024tvirm} & 0.3695$\pm$0.0658 & 0.5161$\pm$0.0835 & -- & -- \\
OOD-TV-Minimax-$\ell_1$~\citep{wang2025oodtvirm} & 0.3628$\pm$0.0654 & 0.4891$\pm$0.0763 & 74.42$\pm$1.57 & 66.38$\pm$1.31 \\
ECTR (ours, inferred-env) & \textbf{0.3452$\pm$0.0589} & \textbf{0.4574$\pm$0.0573} & \textbf{75.24$\pm$1.35} & \textbf{71.22$\pm$1.28} \\
\bottomrule
\end{tabular}
\end{adjustbox}
\medskip
\begin{adjustbox}{max width=\linewidth}
\begin{tabular}{lcccc}
\toprule
Method & Landcover IID & Landcover OOD & Adult Mean & Adult Worst \\
\midrule
\multicolumn{5}{l}{\textit{Observed-environment variants}} \\
IRM~\citep{arjovsky2019irm} & -- & -- & 81.34$\pm$0.45 & 78.76$\pm$0.65 \\
IRM-TV-$\ell_1$~\citep{lai2024tvirm} & -- & -- & 82.31$\pm$0.36 & 79.80$\pm$0.96 \\
OOD-TV-IRM-$\ell_1$~\citep{wang2025oodtvirm} & -- & -- & 83.05$\pm$0.63 & 81.08$\pm$1.09 \\
ECTR (ours) & -- & -- & \textbf{83.45$\pm$0.64} & \textbf{81.31$\pm$0.92} \\
\midrule
\multicolumn{5}{l}{\textit{Methods evaluated without environment labels}} \\
ERM~\citep{vapnik1998ERM} & 59.07$\pm$0.66 & 58.35$\pm$0.91 & 81.16$\pm$0.38 & 78.54$\pm$0.43 \\
EIIL~\citep{creager2021eiil} & 63.98$\pm$1.44 & 60.16$\pm$1.62 & 71.77$\pm$0.72 & 69.94$\pm$0.77 \\
LfF~\citep{nam2020LFF} & 57.61$\pm$2.48 & 54.53$\pm$2.96 & 74.40$\pm$3.12 & 72.03$\pm$3.21 \\
ZIN~\citep{lin2022zin} & 64.31$\pm$1.94 & 62.25$\pm$2.72 & 81.25$\pm$1.07 & 78.67$\pm$1.42 \\
TIVA~\citep{tan2023TIVAtiva} & 55.72$\pm$0.92 & 52.60$\pm$1.40 & 80.95$\pm$0.97 & 78.28$\pm$1.09 \\
Minimax-TV-$\ell_1$~\citep{lai2024tvirm} & 63.77$\pm$1.48 & 63.25$\pm$1.94 & 82.32$\pm$0.78 & 79.85$\pm$0.96 \\
OOD-TV-Minimax-$\ell_1$~\citep{wang2025oodtvirm} & \textbf{64.98$\pm$1.57} & 63.72$\pm$2.13 & 82.53$\pm$0.99 & 80.24$\pm$1.34 \\
ECTR (ours, inferred-env) & 64.59$\pm$1.33 & \textbf{63.80$\pm$1.67} & \textbf{83.56$\pm$0.57} & \textbf{81.15$\pm$0.63} \\
\bottomrule
\end{tabular}
\end{adjustbox}
\end{table}

\begin{table}[t]
\centering
\caption{\textbf{Original evaluation (vision).}
Mean accuracy (\%) is reported as mean$\pm$std over 3 runs; higher is better.}
\label{tab:t2}

\resizebox{\linewidth}{!}{
\setlength{\tabcolsep}{8pt}      
\renewcommand{\arraystretch}{1.25} 
\begin{tabular}{@{}l c c c@{}}
\toprule
 & \textbf{Colored MNIST} & \textbf{PACS} & \textbf{NICO} \\
\textbf{Method} 
& \textbf{Mean} 
& \textbf{Mean} 
& \textbf{Vehicle / Animal (Mean)} \\
\midrule

\multicolumn{4}{@{}l@{}}{\textit{Methods requiring environment annotations}} \\

IRM~\citep{arjovsky2019irm} 
& 79.50$\pm$1.11 
& 81.50$\pm$2.20 
& 65.52$\pm$0.64 / 76.11$\pm$1.55 \\

groupDRO~\citep{sagawa2020groupdro} 
& -- 
& 83.50$\pm$1.40 
& 71.93$\pm$1.86 / 87.06$\pm$0.52 \\

\textbf{ECTR (ours)}
& \textbf{97.31$\pm$0.62} 
& \textbf{84.70$\pm$1.40} 
& \textbf{85.68$\pm$0.52 / 89.91$\pm$1.10} \\

\midrule
\multicolumn{4}{@{}l@{}}{\textit{Methods without environment annotations}} \\

EIIL~\citep{creager2021eiil} 
& 82.36$\pm$1.32 
& -- 
& 69.57$\pm$1.64 / 86.45$\pm$1.73 \\

ZIN~\citep{lin2022zin} 
& 95.02$\pm$1.70 
& -- 
& 73.66$\pm$1.77 / 87.30$\pm$1.18 \\

OOD-TV-Minimax-$\ell_1$~\citep{wang2025oodtvirm} 
& 95.14$\pm$0.95 
& 83.60$\pm$1.40 
& 72.89$\pm$0.85 / 88.26$\pm$0.72 \\

\textbf{ECTR (ours, inferred-env)}
& \textbf{96.59$\pm$0.79} 
& \textbf{85.10$\pm$1.60}
& \textbf{84.24$\pm$1.13 / 88.39$\pm$0.67} \\

\bottomrule
\end{tabular}
}
\end{table}

With observed environments, ECTR achieves House mean/worst MSE of $0.3958/0.5461$, CelebA worst accuracy of $73.58$, and Adult mean/worst accuracy of $83.45/81.31$. These are the best means for the respective metrics in that information setting. With inferred environments, ECTR achieves House mean/worst MSE of $0.3452/0.4574$, CelebA mean/worst accuracy of $75.24/71.22$, and Adult mean/worst accuracy of $83.56/81.15$, leading the corresponding metrics among methods evaluated without training environment labels. Landcover OOD accuracy is $63.80$. Table~\ref{tab:t2} extends the comparisons to Colored MNIST, PACS, and NICO using the protocols in Appendix~\ref{app:protocols}.
\FloatBarrier

\subsection{Risk-coupling and normalization controls}
\label{app:ablation}
For $\mathcal I_e=\{i:g_i=e\}$, $n_e=|\mathcal I_e|$, and $\rho_e=n_e/N$, the full construction uses Eq.~\eqref{eq:objective}. All scalar stationarity probes in the controlled ablation use L1.

\paragraph{Global tail uses global weights and a global KL reference.}
The alternative global-weighted environment risk is
\[
R^{\mathrm{glob}}_{e,\theta}(w)=\frac{1}{\rho_e}\sum_{i\in\mathcal I_e}\pi_\theta(i)\ell_i(w).
\]
Uniform global weights recover the ordinary environment risk. Its objective is
\[
\mathcal J_{\mathrm{glob}}
=\mathbb E_e[R^{\mathrm{glob}}_{e,\theta}(1)]
+\lambda\mathbb E_e\!\left[\left|\left.\nabla_wR^{\mathrm{glob}}_{e,\theta}(w)\right|_{w=1}\right|\right]
-\beta\KL(\pi_\theta\|\mathrm{Unif}_N).
\]
This changes both normalization and the KL reference relative to Full; the comparison tests these implemented constructions together.

\paragraph{Main-only weights prediction and probes uniform risk.}
Define $\bar R_e(w)=n_e^{-1}\sum_{i\in\mathcal I_e}\ell_i(w)$. Then
\[
\mathcal J_{\mathrm{main}}
=\mathbb E_e[R_{e,\theta}(1)]
+\lambda\mathbb E_e\!\left[\left|\left.\nabla_w\bar R_e(w)\right|_{w=1}\right|\right]
-\beta\KLenv.
\]
Prediction is weighted, TV is evaluated on uniform risk, and KL is unchanged. No-KL uses weighted prediction and weighted TV risks with $\beta=0$.

\paragraph{Weight diagnostics quantify stationarity and concentration.}
Tail-TV Eval is $\mathbb E_e[|\nabla_w R_{e,\theta}|_{w=1}]$, evaluated with the weighting construction of the model. For nonnegative relative weights $\omega_i$, the ESS ratio and Top-5 mass are
\[
\frac{\ESS(\omega)}{N}=\frac{(\sum_i\omega_i)^2}{N\sum_i\omega_i^2},
\qquad
\mathrm{Top5Mass}=\frac{\sum_{i\in\mathcal T_{5\%}}\omega_i}{\sum_i\omega_i}.
\]
These quantities describe the dispersion of learned weights. The experiment also tabulates Env Mass Skew, with the source definition
\[
\mathrm{Std}_e\left(\frac{\sum_{i:g_i=e}\pi_\theta(i)}{n_e/N}\right).
\]
Table~\ref{tab:ectr_ablation_mixed_compact} gives the complete component ablation; all accuracy and diagnostic entries are means over three random seeds, as in Table~\ref{tab:core}. Full ECTR combines higher mean worst-group accuracy with more dispersed weights than Main-only and No-KL across the three settings.

\begin{table}[htbp]
\centering
\footnotesize
\setlength{\tabcolsep}{3.6pt}
\caption{Original evaluation: complete component ablation on mixed \texttt{logit\_z}. All results are means over three random seeds. Accuracy, Tail-TV, ESS, and Top-5 mass repeat Table~\ref{tab:core}; Env Mass Skew is additionally reported. Diagnostic definitions are in Appendix~\ref{app:ablation}.}
\label{tab:ectr_ablation_mixed_compact}
\begin{adjustbox}{max width=\linewidth}
\begin{tabular}{llccccc}
\toprule
Setting & Method 
& Worst Acc $\uparrow$ 
& Env Mass Skew 
& Tail-TV Eval 
& ESS Ratio 
& Top-5 Mass \\
\midrule
Hard
& Global tail & 64.06 & 0.1225 & 0.000592 & 0.689 & 0.126 \\
& Main only   & 60.52 & 0.0000 & 0.075835 & 0.079 & 0.683 \\
& No KL       & 62.95 & 0.0000 & 0.000625 & 0.408 & 0.283 \\
& \textbf{Full ECTR} 
              & 68.44 & 0.0000 & 0.000505 & 0.699 & 0.121 \\
\midrule
Medium
& Global tail & 68.29 & 0.0928 & 0.000586 & 0.703 & 0.120 \\
& Main only   & 65.00 & 0.0000 & 0.065626 & 0.085 & 0.659 \\
& No KL       & 68.40 & 0.0000 & 0.000594 & 0.519 & 0.198 \\
& \textbf{Full ECTR} 
              & 70.48 & 0.0000 & 0.000509 & 0.702 & 0.120 \\
\midrule
Easy
& Global tail & 72.47 & 0.0801 & 0.000626 & 0.710 & 0.117 \\
& Main only   & 70.32 & 0.0000 & 0.041404 & 0.111 & 0.575 \\
& No KL       & 72.19 & 0.0000 & 0.000622 & 0.588 & 0.186 \\
& \textbf{Full ECTR} 
              & 73.56 & 0.0000 & 0.000530 & 0.714 & 0.118 \\
\bottomrule
\end{tabular}
\end{adjustbox}
\end{table}

\FloatBarrier

\subsection{Concentration sensitivity}
\label{app:original-sensitivity}
These studies use the Original Hard mixed setting from Table~\ref{tab:app-controlled-mixed-settings} and report means over five random seeds. Table~\ref{tab:lr2_sensitivity_logitz} varies the multiplier learning rate; Table~\ref{tab:lr3_sensitivity_logitz} varies the tail learning rate; Table~\ref{tab:conditional_beta_collapse_recovery} fixes $\eta_\theta=10^{-2}$, $\eta_\Psi=10^{-4}$, and $\eta_\Phi=10^{-4}$ and varies KL strength.

\begin{table}[htbp]
\centering
\caption{Original evaluation: 
Sensitivity to the TV-multiplier learning rate \(\eta_{\Psi}\) on \texttt{LOGIT-Z}.
Results are averaged over five random seeds.
}
\label{tab:lr2_sensitivity_logitz}
\begin{tabular}{lcc}
\toprule
\(\eta_{\Psi}\) 
& Worst Acc Mean 
& Global ESS Ratio Mean \\
\midrule
\(10^{-5}\) 
& \(0.6148\) 
& \(0.3740\) \\

\(10^{-4}\) 
& \(0.6133\) 
& \(0.3740\) \\

\(10^{-3}\) 
& \(0.5980\) 
& \(0.3740\) \\

\(10^{-2}\) 
& \(0.6114\) 
& \(0.3739\) \\
\bottomrule
\end{tabular}
\end{table}

\begin{table}[htbp]
\centering
\caption{Original evaluation: 
Sensitivity to the tail-adversary learning rate \(\eta_{\theta}\) on \texttt{LOGIT-Z}.
Results are averaged over five random seeds.
}
\label{tab:lr3_sensitivity_logitz}
\resizebox{\linewidth}{!}{
\begin{tabular}{lcccc}
\toprule
\(\eta_{\theta}\) 
& Tail-TV Eval Mean
& Top-5 Mass Mean
& Worst Acc Mean
& Global ESS Ratio Mean \\
\midrule
\(10^{-4}\)
& \(0.00134\)
& \(0.1059\)
& \(0.6964\)
& \(0.8316\) \\

\(10^{-3}\)
& \(0.00178\)
& \(0.3361\)
& \(0.6016\)
& \(0.2898\) \\

\(10^{-2}\)
& \(0.3240\)
& \(1.0000\)
& \(0.5302\)
& \(0.0006\) \\
\bottomrule
\end{tabular}}
\end{table}

\begin{table}[htbp]
\centering
\caption{Original evaluation: 
Conditional sensitivity of ECTR to the environment-wise KL coefficient \(\beta\) under an aggressive tail adversary.
We fix \(\eta_\theta=10^{-2}\), \(\eta_\Psi=10^{-4}\), and \(\eta_\Phi=10^{-4}\).
Small \(\beta\) values lead to severe weight collapse, while sufficiently large \(\beta\) recovers a high-ESS, non-collapsed weighting distribution.
}
\label{tab:conditional_beta_collapse_recovery}
\begin{tabular}{lccc}
\toprule
\(\beta\)
& Worst Acc
& ESS Ratio
& Top-5 Mass \\
\midrule
\(0.01\)
& \(0.5428\)
& \(0.0006\)
& \(1.0000\) \\

\(0.10\)
& \(0.5400\)
& \(0.0005\)
& \(1.0000\) \\

\(0.50\)
& \(0.5268\)
& \(0.0010\)
& \(0.9911\) \\

\(1.00\)
& \(0.5380\)
& \(0.0015\)
& \(0.9748\) \\

\(5.00\)
& \(0.6892\)
& \(0.9988\)
& \(0.0526\) \\

\(10.00\)
& \(0.6900\)
& \(0.9994\)
& \(0.0522\) \\
\bottomrule
\end{tabular}
\end{table}

At $\eta_\theta=10^{-2}$ the Original tail-rate sweep has ESS $0.0006$ and Top-5 mass $1.0000$. In the conditional-$\beta$ sweep, $\beta=5.00$ recovers ESS $0.9988$ and worst accuracy $0.6892$. The multiplier-rate sweep maintains nearly unchanged global ESS. Together, the sweeps show how KL strength and the adversary learning rate regulate weight concentration.
\FloatBarrier

\section{Extended evaluation}
\label{app:extended}
The extended evaluation tests invariant-learning baselines, simple tails under corruption, environment information, concentration, aggregation, and training cost. Each study uses its own protocol and uncertainty convention, specified below.

\subsection{Invariant-learning comparisons}
\label{app:baselines}
V-REx, RDM, WRI, and ECTR are compared using the five-seed protocol in Appendix~\ref{app:protocols}. LOGIT-Z uses mixed-hard. ECTR has the best mean for each metric in Table~\ref{tab:extended-baselines}, with the largest differences on House and CelebA. WRI uses frozen representations for the classification tasks.

\begin{table}[htbp]
\centering
\caption{Extended evaluation: invariant-learning baselines, five prespecified seeds, with means and accompanying uncertainty. Accuracy is in percent; House Price uses MSE. WRI uses frozen-representation classification; a dash marks an unavailable regression entry. Hyperparameters were selected on clean validation on seed 1, then fixed across the five formal seeds.}
\label{tab:extended-baselines}
\footnotesize\setlength{\tabcolsep}{3pt}
\begin{adjustbox}{max width=\linewidth}
\begin{tabular}{llcccc}
\toprule
Dataset & Metric & V-REx & RDM & WRI & ECTR \\
\midrule
LOGIT-Z & Test Acc. \(\uparrow\) & 42.38\(\pm\)22.57 & 42.36\(\pm\)22.53 & 25.95\(\pm\)5.10 & 44.66\(\pm\)22.35 \\
LOGIT-Z & Worst-env Acc. \(\uparrow\) & 39.78\(\pm\)23.84 & 39.87\(\pm\)23.85 & 20.16\(\pm\)4.61 & 42.25\(\pm\)22.05 \\
House Price & Test MSE \(\downarrow\) & 0.5567\(\pm\)0.0603 & 0.6434\(\pm\)0.0947 & - & 0.4152\(\pm\)0.0920 \\
House Price & Worst-env MSE \(\downarrow\) & 0.7316\(\pm\)0.0906 & 0.8060\(\pm\)0.1500 & - & 0.5986\(\pm\)0.1182 \\
Colored MNIST & Test Acc. \(\uparrow\) & 97.21\(\pm\)0.19 & 97.03\(\pm\)0.13 & 96.65\(\pm\)0.30 & 97.40\(\pm\)0.69 \\
Colored MNIST & Worst-env Acc. \(\uparrow\) & 96.99\(\pm\)0.17 & 96.86\(\pm\)0.11 & 96.46\(\pm\)0.35 & 97.23\(\pm\)0.56 \\
CelebA & Test Acc. \(\uparrow\) & 70.80\(\pm\)0.24 & 69.82\(\pm\)0.38 & 70.07\(\pm\)0.63 & 77.56\(\pm\)1.83 \\
CelebA & Worst-env Acc. \(\uparrow\) & 60.08\(\pm\)0.34 & 57.89\(\pm\)0.75 & 58.41\(\pm\)1.25 & 73.93\(\pm\)2.42 \\
\bottomrule
\end{tabular}
\end{adjustbox}
\end{table}

\paragraph{Paired test-accuracy gains.}
For LOGIT-Z Test accuracy, the V-REx, RDM, and ECTR absolute summaries in Table~\ref{tab:extended-baselines} use population standard deviations across five runs (divisor $5$, or ddof $=0$). The paired summaries below instead use the sample standard deviation of the five accuracy differences (divisor $4$, or ddof $=1$). With accuracies measured in percent, let $d_s=A_{\mathrm{ECTR},s}-A_{\mathrm{baseline},s}$, so gains are in percentage points (pp):
\[
\bar d=\frac15\sum_{s=1}^5d_s,\qquad
s_d=\sqrt{\frac14\sum_{s=1}^5(d_s-\bar d)^2}.
\]
\begin{center}
\begin{tabular}{lccc}
\toprule
Comparison & Mean paired gain (pp) & Sample SD (pp) & Positive pairs\\
\midrule
ECTR minus V-REx & $+2.28$ & $1.39$ & $5/5$\\
ECTR minus RDM & $+2.30$ & $1.35$ & $5/5$\\
\bottomrule
\end{tabular}
\end{center}
These descriptive summaries concern matched data, splits, and run conditions within each seed in the Hard configuration of Table~\ref{tab:app-controlled-mixed-settings}. Across seeds, the data realization, split, and training randomness vary together. Absolute performance moves together across methods, while the paired gains are more concentrated and positive in all five comparisons for each baseline. The sample SD measures the dispersion of paired gains, not a standard error or confidence interval. These conventions apply specifically to the three methods' LOGIT-Z Test accuracy and its paired differences; other metrics retain their own uncertainty conventions.

\FloatBarrier

\subsection{Weighting under training corruption}
\label{app:noise}
All methods share data, environments, architecture, $\Phi/\Psi$ initialization, budget, corruption masks, and checkpoint rule. Each uses its own within-environment weights $\omega_{i\mid e}$ in both $R_{e,\omega}=\sum_{i:g_i=e}\omega_{i\mid e}\ell_i$ and the L1 stationarity penalty. Uniform assigns $1/n_e$, has zero KL, and makes no sample-adversary update. Top-$\alpha$ assigns $1/k_e$ to the $k_e=\max(1,\lceil\alpha n_e\rceil)$ largest-loss samples. For empirical CVaR~\citep{rockafellar2002CVaR}, let $m_e=\alpha n_e$: the largest $\lfloor m_e\rfloor$ losses receive weight $1/m_e$ and the next receives the remaining fractional mass. The two rules coincide when $\alpha n_e$ is an integer. Ranking uses detached losses; selected original losses retain gradients.

The fixed fractions are $\alpha=0.708$ for LOGIT-Z, $0.400$ for CelebA, and $0.600$ for Adult, chosen by approximate clean-training ESS matching and frozen across seeds and corruption levels. The House Price ZIN protocol uses the fixed $\alpha=0.700$ setting inherited from its protocol. Full ECTR learns sample weights with $\theta$, conditions them on environments, and maximizes the weighted prediction and TV terms with $-\beta\KLenv$, as in Eq.~\eqref{eq:objective}.

\paragraph{Corruption affects training labels and targets.}
Only training labels/targets are corrupted; validation and test remain clean. LOGIT-Z independently selects $\lfloor pn_e\rfloor$ examples within each environment and flips binary labels. House selects the same proportion and circularly permutes selected targets within their environment; normalization statistics are computed from clean training targets and fixed. CelebA flips Smiling within environment, clean-label, and Male-attribute strata. Adult flips income labels within environment and clean-label strata. Methods use identical nested masks for a given dataset, seed, and noise rate.

\paragraph{NMR quantifies the mass assigned to corrupted samples.}
For corrupted indices $\mathcal N_e$,
\[
\mathrm{NMR}_e=\frac{\sum_{i\in\mathcal N_e}\omega_{i\mid e}}{|\mathcal N_e|/n_e}.
\]
NMR one is proportional weighting; values above one mean excess corrupted-sample mass. Table~\ref{tab:extended-nmr} contains all evaluated nonzero-noise conditions. ECTR's mean NMR is below both fixed tails in each listed condition. For all prespecified conditions up to $p=.2$, that paired direction holds in $5/5$ seeds.

\begin{table}[htbp]
\centering
\caption{Extended evaluation: noise mass ratio (NMR), five-seed mean $\pm$ sample standard deviation. All listed rates are nonzero training corruption. Uniform NMR is one; larger values indicate excess mass on corrupted samples.}
\label{tab:extended-nmr}
\footnotesize\setlength{\tabcolsep}{3pt}
\begin{adjustbox}{max width=\linewidth}
\begin{tabular}{llcccc}
\toprule
Dataset & Noise & Uniform & Top-$\alpha$ & CVaR & Full ECTR \\
\midrule
LOGIT-Z & .1 & $(1.0000\pm0.0000)$ & $(1.2733\pm0.1024)$ & $(1.2733\pm0.1024)$ & $(0.9893\pm0.0238)$ \\
LOGIT-Z & .2 & $(1.0000\pm0.0000)$ & $(1.2549\pm0.1208)$ & $(1.2549\pm0.1208)$ & $(1.0004\pm0.0319)$ \\
LOGIT-Z & .4 & $(1.0000\pm0.0000)$ & $(1.0964\pm0.1555)$ & $(1.0964\pm0.1555)$ & $(1.0073\pm0.0086)$ \\
House Price & .1 & $(1.0000\pm0.0000)$ & $(1.2636\pm0.0120)$ & $(1.2681\pm0.0144)$ & $(1.0262\pm0.0828)$ \\
House Price & .2 & $(1.0000\pm0.0000)$ & $(1.2168\pm0.0200)$ & $(1.2185\pm0.0200)$ & $(1.0189\pm0.0528)$ \\
House Price & .4 & $(1.0000\pm0.0000)$ & $(1.1300\pm0.0064)$ & $(1.1281\pm0.0066)$ & $(1.0053\pm0.0218)$ \\
CelebA & .1 & $(1.0000\pm0.0000)$ & $(2.0930\pm0.2364)$ & $(2.0281\pm0.2327)$ & $(0.9932\pm0.0164)$ \\
CelebA & .2 & $(1.0000\pm0.0000)$ & $(1.7930\pm0.3252)$ & $(1.7782\pm0.3286)$ & $(0.9981\pm0.0138)$ \\
Adult & .2 & $(1.0000\pm0.0000)$ & $(1.6122\pm0.0111)$ & $(1.6034\pm0.0098)$ & $(1.0338\pm0.0444)$ \\
\bottomrule
\end{tabular}
\end{adjustbox}
\end{table}

\paragraph{Predictive comparisons complement the weighting diagnostics.}
Table~\ref{tab:extended-noise} gives the complete predictive outcomes. On LOGIT-Z, ECTR's worst-environment mean exceeds the fixed tails by $3.87$ and $4.86$ percentage points at $p=.1/.2$. At CelebA $p=.1$, its improvements over Top-$\alpha$ and CVaR are $7.02$ and $5.30$ points, respectively, both favoring ECTR in $5/5$ seeds. On House, ECTR has the lowest mean MSE at $.1/.2$. The table also includes the higher corruption rates and Adult comparison with their seed-level sample standard deviations.

\begin{table}[htbp]
\centering
\caption{Extended evaluation: predictive performance under nonzero training corruption, five-seed mean $\pm$ sample standard deviation. Accuracy remains on the reported 0--1 scale; House uses MSE. Uniform denotes Uniform TV-IRM; WE/WG denote worst environment/group.}
\label{tab:extended-noise}
\footnotesize\setlength{\tabcolsep}{3pt}
\begin{adjustbox}{max width=\linewidth}
\begin{tabular}{llccccc}
\toprule
Dataset & Noise & Metric & Uniform & Top-$\alpha$ & CVaR & Full ECTR \\
\midrule
LOGIT-Z & .1 & WE Acc. ($\uparrow$) & $(0.4063\pm0.2568)$ & $(0.3827\pm0.2377)$ & $(0.3827\pm0.2377)$ & $(0.4214\pm0.2158)$ \\
LOGIT-Z & .2 & WE Acc. ($\uparrow$) & $(0.4069\pm0.2468)$ & $(0.3692\pm0.2183)$ & $(0.3692\pm0.2183)$ & $(0.4178\pm0.2105)$ \\
LOGIT-Z & .4 & WE Acc. ($\uparrow$) & $(0.4034\pm0.2215)$ & $(0.4077\pm0.2130)$ & $(0.4077\pm0.2130)$ & $(0.4086\pm0.2217)$ \\
House Price & .1 & WE MSE ($\downarrow$) & $(0.7055\pm0.1677)$ & $(0.7033\pm0.1694)$ & $(0.7157\pm0.1675)$ & $(0.6718\pm0.1770)$ \\
House Price & .2 & WE MSE ($\downarrow$) & $(0.7004\pm0.0530)$ & $(0.6953\pm0.0445)$ & $(0.6841\pm0.0488)$ & $(0.6360\pm0.1567)$ \\
House Price & .4 & WE MSE ($\downarrow$) & $(0.6133\pm0.1097)$ & $(0.6135\pm0.1321)$ & $(0.6254\pm0.1352)$ & $(0.6428\pm0.0949)$ \\
CelebA & .1 & WE Acc. ($\uparrow$) & $(0.5031\pm0.0123)$ & $(0.4445\pm0.0361)$ & $(0.4617\pm0.0216)$ & $(0.5147\pm0.0200)$ \\
CelebA & .2 & WE Acc. ($\uparrow$) & $(0.4952\pm0.0066)$ & $(0.5010\pm0.0416)$ & $(0.4912\pm0.0327)$ & $(0.4987\pm0.0227)$ \\
Adult & .2 & WG Acc. ($\uparrow$) & $(0.8058\pm0.0164)$ & $(0.7826\pm0.0236)$ & $(0.7785\pm0.0241)$ & $(0.8051\pm0.0246)$ \\
\bottomrule
\end{tabular}
\end{adjustbox}
\end{table}

LOGIT-Z Tail-TV values at $p=.1/.2$ are $0.0291/0.0254$ for ECTR and $0.0628/0.0724$ for fixed tails. Each probe is computed using its method's within-environment weights, as defined in Appendix~\ref{app:ablation}.
\FloatBarrier

\subsection{Environment-information controls}
\label{app:information}
The five-seed study uses two protocols. End-to-end degradation perturbs the auxiliary inputs and relearns $q_\eta$ during model training. House conditions include full auxiliary information, masking $50\%$, shuffled auxiliary inputs, and zero auxiliary inputs. Shared-$q$ controls instead directly construct and freeze the memberships supplied identically to TV-IRM and ECTR for each seed and condition, with no updates to the environment inference network. This fixes the environment assignments for the objective comparison.

\begin{table}[htbp]
\centering
\caption{Extended evaluation: end-to-end auxiliary degradation, five-seed study. House entries are MSE; Adult entries are accuracy. Uncertainty is retained where reported; Adult summaries provide means only. Worst refers to the reported worst-environment or worst-group metric.}
\label{tab:extended-auxiliary}
\footnotesize\setlength{\tabcolsep}{3pt}
\begin{adjustbox}{max width=\linewidth}
\begin{tabular}{llccc}
\toprule
Dataset / condition & TV-IRM Overall & TV-IRM Worst & ECTR Overall & ECTR Worst \\
\midrule
House: known env & 0.4727\(\pm\)0.0820 & 0.6193\(\pm\)0.0995 & 0.4152\(\pm\)0.0920 & 0.5986\(\pm\)0.1182 \\
House: full auxiliary & 0.4821\(\pm\)0.1321 & 0.5717\(\pm\)0.1127 & 0.4078\(\pm\)0.1243 & 0.5404\(\pm\)0.1568 \\
House: mask 50\% & 0.5604\(\pm\)0.1297 & 0.6861\(\pm\)0.1729 & 0.4639\(\pm\)0.1174 & 0.5503\(\pm\)0.1469 \\
House: shuffled & 0.7010\(\pm\)0.1415 & 0.8629\(\pm\)0.1867 & 0.5120\(\pm\)0.1788 & 0.7829\(\pm\)0.2609 \\
House: zero auxiliary & 0.6761\(\pm\)0.0729 & 0.8092\(\pm\)0.1055 & 0.5486\(\pm\)0.1927 & 0.8471\(\pm\)0.2575 \\
Adult: full auxiliary & 83.19\% & 80.92\% & 84.36\% & 81.96\% \\
Adult: fully shuffled & 83.19\% & 80.92\% & 84.44\% & 82.12\% \\
\bottomrule
\end{tabular}
\end{adjustbox}
\end{table}

House uses construction year, while observed environments are decade bins. With full auxiliary information, overall MSE is $0.4078$ for ECTR and $0.4821$ for TV-IRM; with shuffled input, it is $0.5120$ and $0.7010$, respectively. Adult uses the six numerical covariates in Appendix~\ref{app:protocols} as auxiliary inputs. After full row corruption, NMI changes from $0.0039$ to $0.0001$ for TV-IRM and from $0.0068$ to $0.0002$ for ECTR, while overall and worst-group accuracy remain nearly unchanged.

\paragraph{Fixed memberships define the shared-$q$ conditions.}
Oracle memberships are $q_{\mathrm{oracle},i}=\operatorname{onehot}(g_i)$, while uniform memberships assign $1/E$ to every environment. House soft $0.50$ uses
\[
q_i=0.50q_{\mathrm{oracle},i}+0.50(1/E,\ldots,1/E).
\]
House shuffled permutes the rows of oracle $q$ across samples, retaining the membership vectors and overall environment proportions while disrupting sample--environment correspondence. Adult uses $E=4$ race-group--sex environments in this shared-$q$ protocol. Its fully permuted condition likewise applies a sample-level permutation to all oracle one-hot rows, preserving the four columns and overall environment proportions. These memberships remain frozen and identical for both objectives throughout training.

\begin{table}[htbp]
\centering
\caption{Extended evaluation: mean results with identical frozen memberships supplied to TV-IRM and ECTR in the five-seed environment-information study. House uses MSE and Adult uses accuracy; oracle refers to the observed partition.}
\label{tab:extended-membership}
\footnotesize\setlength{\tabcolsep}{3pt}
\begin{adjustbox}{max width=\linewidth}
\begin{tabular}{llccc}
\toprule
Dataset / (q) condition & TV-IRM Overall & ECTR Overall & TV-IRM Worst & ECTR Worst \\
\midrule
House: oracle & 0.4727 & 0.4152 & 0.6193 & 0.5986 \\
House: soft 0.50 & 0.4689 & 0.3919 & 0.6100 & 0.5639 \\
House: uniform & 0.6933 & 0.5407 & 0.8856 & 0.7354 \\
House: shuffled & 0.7901 & 0.5267 & 1.0510 & 0.6823 \\
Adult: oracle & 84.25\% & 84.27\% & 80.92\% & 81.65\% \\
Adult: fully permuted & 84.24\% & 84.40\% & 80.85\% & 81.83\% \\
Adult: uniform & 84.29\% & 84.85\% & 80.92\% & 81.64\% \\
\bottomrule
\end{tabular}
\end{adjustbox}
\end{table}

All listed shared-$q$ mean comparisons favor ECTR. On House, soft $0.50$ gives lower MSE than the hard observed partition. On Adult, ECTR's means remain similar across the tested assignments. With the same frozen memberships supplied to both methods, these comparisons support the benefit of the coupled objective under matched environment information.
\FloatBarrier

\subsection{Concentration-control settings}
\label{app:stability}
Within each seed, configurations share data; hyperparameter selection uses validation performance. The extended Medium $(\eta_\theta,\beta)$ grid contains $63$ healthy configurations. Healthy requires no collapse across seeds, mean minimum-environment ESS at least $0.20$, and mean Top-5 mass below $0.50$. Of these, $15$ lie within $0.01$ worst-validation accuracy of the best healthy configuration, forming two connected regions.

\begin{table}[htbp]
\centering
\caption{Extended evaluation: tail-adversary/KL sensitivity, five seeds per configuration. The anchors and plateau concern Medium mixed-shift LOGIT-Z; the last two rows are Hard stress tests. Worst-Val is worst-validation accuracy. Displayed ESS is global, whereas the healthy criterion uses minimum-environment ESS.}
\label{tab:extended-grid}
\footnotesize\setlength{\tabcolsep}{3pt}
\begin{adjustbox}{max width=\linewidth}
\begin{tabular}{llccccc}
\toprule
Configuration & $\eta_\theta$ & $\beta$ & Worst-Val & Global ESS & Top-5\% & \shortstack{Collapse\\runs} \\
\midrule
\shortstack[l]{Default-neighborhood\\anchors (5 configs)} & $(3\times10^{-5})$--$(3\times10^{-4})$ & (0)--(0.01) & 0.7014--0.7125 & 0.6853--0.7052 & 0.1371--0.1444 & 0/25 \\
\shortstack[l]{Largest healthy mean\\plateau (8 configs)} & $(10^{-3})$ & (0)--(5) & 0.7207--0.7270 & 0.4251--0.8541 & 0.1044--0.2362 & 0/40 \\
\shortstack[l]{Hard, insufficient\\KL} & $(10^{-2})$ & (0.05) & 0.6525 & 0.0009 & 0.9775 & 5/5 \\
\shortstack[l]{Hard, strong\\KL} & $(10^{-2})$ & (10) & 0.7175 & 0.9479 & 0.0767 & 0/5 \\
\bottomrule
\end{tabular}
\end{adjustbox}
\end{table}

The default-neighborhood anchors span $3\times10^{-5}$--$3\times10^{-4}$ and $\beta=0$--$0.01$. The eight-configuration plateau at $\eta_\theta=10^{-3}$ spans $\beta=0$--$5$, with reported variation $0.0062$ and rounded endpoint values in the table. Under Hard stress at $\eta_\theta=10^{-2}$, the strong-KL sweep $\beta\in\{5,10,20\}$ has $0/5$ collapse runs at each value; $\beta=10$ is the numerical representative in Table~\ref{tab:extended-grid}.

The effective coefficient is $\gamma_t\lambda(\Psi,\Phi)$, with maximum annealing scale $\gamma_{\max}$. This differs from the multiplier-player learning rate $\eta_\Psi$. Table~\ref{tab:extended-gamma} retains the effective-strength sweep.

\begin{table}[htbp]
\centering
\caption{Extended evaluation: effective TV strength, five-seed study. LOGIT-Z uses correlation-only shift. Validation and test metrics, ESS, and collapse counts describe behavior across the effective-TV range.}
\label{tab:extended-gamma}
\footnotesize\setlength{\tabcolsep}{3pt}
\begin{adjustbox}{max width=\linewidth}
\begin{tabular}{p{62pt}cp{80pt}p{80pt}cc}
\toprule
Dataset & $\gamma_{\max}$ & Validation metric & Test metric & ESS & Collapse \\
\midrule
\shortstack[l]{LOGIT-Z\\correlation-only} & (1)--(100) & Worst Acc: 0.6998--0.7040 & Worst Acc: 0.3550--0.3562 & 0.6749--0.6795 & 0/5 \\
House Price & (0.5)--(10) & Worst MSE: 0.3534--0.3543 & Worst MSE: 0.6114--0.7515 & 0.6620 & 0/5 \\
\bottomrule
\end{tabular}
\end{adjustbox}
\end{table}

LOGIT-Z worst-test accuracy spans $0.3550$--$0.3562$ with $0/5$ collapse throughout the tested range. House has nearly constant validation MSE and ESS, with worst-test MSE spanning $0.6114$--$0.7515$. These measurements describe the response to the effective TV coefficient separately for validation, test performance, and weight dispersion.
\FloatBarrier

\subsection{Outer aggregation}
\label{app:aggregation}
The alternatives are
\[
\rho_e^{\mathrm{uniform}}=1/E,\qquad
\rho_e^{\mathrm{size}}=\frac{n_e}{\sum_jn_j},\qquad
\rho_e^{\mathrm{adaptive}}=\softmax_e\!\left(\frac{\operatorname{EMA}(R_{e,\theta})}{\tau}\right).
\]
The same $\rho_e$ is applied to prediction, TV, and KL, changing only the outer measure. Known-environment LOGIT-Z Medium is evaluated with five paired seeds in four settings: balanced clean, small difficult environment, large difficult environment, and balanced environments with $20\%$ label flips in one environment. Adaptive temperature is selected on worst-environment validation using three seeds, then fixed to $\tau=0.05$. Shared seed, data, initialization, and checkpoint protocol support paired differences.

\begin{table}[htbp]
\centering
\caption{Extended evaluation: outer aggregation on LOGIT-Z Medium, five paired seeds. Uniform Worst is absolute worst-environment accuracy in percent. Alternatives are paired differences from Uniform in percentage points, not absolute accuracies. Reported uncertainty is retained.}
\label{tab:extended-aggregation}
\footnotesize\setlength{\tabcolsep}{3pt}
\begin{adjustbox}{max width=\linewidth}
\begin{tabular}{llcc}
\toprule
Setting & Uniform Worst & Size minus Uniform & Adaptive minus Uniform \\
\midrule
Balanced clean & 39.94 \(\pm\) 28.45 & 0.000\(\pm\)0.000 & -0.104 \(\pm\) 0.234 \\
Difficult environment small & 39.89 \(\pm\) 28.42 & +0.032 \(\pm\) 0.104 & -0.064 \(\pm\) 0.166 \\
Difficult environment large & 39.94 \(\pm\) 28.50 & +0.024 \(\pm\) 0.022 & +0.008 \(\pm\) 0.018 \\
One noisy environment & 39.86 \(\pm\) 28.60 & 0.000\(\pm\)0.000 & -0.312 \(\pm\) 0.653 \\
\bottomrule
\end{tabular}
\end{adjustbox}
\end{table}

In clean settings adaptive weighting assigns approximately $0.58$--$0.61$ mass to the more difficult environment, while normalized environment-weight entropy exceeds $0.90$. Minimum within-environment ESS remains approximately $0.68$, with no environment- or sample-level collapse in these runs. The alternatives therefore use different outer weights while achieving similar predictive performance.

In the noisy setting, sample-level NMR is $0.9836$ for Uniform and $0.9837$ for Adaptive. The outer-weighted corruption contribution is $0.98$ for Uniform and $1.13$ for Adaptive, relative to the uniform reference. Alongside the clean-setting comparisons, these results support uniform aggregation as a simple default for this study.

\subsection{Training and deployment cost}
\label{app:cost}
House Price uses the same $15\to32\to1$ predictor, splits, full-batch budget, hardware, and FP32 precision for both methods, with five independent seeded runs. The extra sample adversary adds $545$ training parameters, reported as less than $3\%$ of the matched TV-IRM training parameter count. The additional graph includes tail-conditioned higher-order TV computation.

\begin{table}[htbp]
\centering
\caption{Extended evaluation: House Price runtime and memory, averaged over five independent runs with distinct seeds. Ratios compare ECTR with environment-matched TV-IRM. Parameter and latency pairs list ECTR then TV-IRM.}
\label{tab:extended-cost}
\footnotesize\setlength{\tabcolsep}{3pt}
\begin{adjustbox}{max width=\linewidth}
\begin{tabular}{llccc}
\toprule
Environment setting & ECTR / TV-IRM time & ECTR / TV-IRM peak memory & Deploy parameters & Batch-1 inference \\
\midrule
Known & $1.340\times$ & $1.034\times$ & 545 vs. 545 & 0.0493 vs. 0.0490 ms \\
Inferred & $1.257\times$ & $1.033\times$ & 545 vs. 545 & 0.0492 vs. 0.0489 ms \\
\bottomrule
\end{tabular}
\end{adjustbox}
\end{table}

Known/inferred training-time ratios correspond to $34.0\%/25.7\%$ overhead. Peak allocated memory ratios are $1.034/1.033$. All training-only adversaries are removed at inference, leaving the same $545$-parameter predictor. Timing coefficients of variation are below $1.8\%$ in the five runs. The House measurements quantify the additional training cost while showing matched deployment size and batch-1 latency.
\FloatBarrier

\section{Within-environment weight distributions}
\label{app:diagnostics}
The Original LOGIT-Z plots compare spurious-aligned majority A and spurious-conflicting minority B. Relative weight one is the within-environment uniform baseline. Density plots clip values at five for display; CCDF plots show the right tail. Across Easy, Medium, and Hard settings, a subset of spurious-conflicting samples receives larger relative weights, producing a heavier right tail for B. The density plots also show substantial mass around and below the uniform baseline.

These run-level diagnostics illustrate selective emphasis within environments. Together with the component ablation, they connect the shared-risk construction to the weighting of difficult spurious-conflicting examples. Appendix~\ref{app:noise} separately measures weight allocation to corrupted training samples through NMR.

\begin{figure}[htbp]
\centering
\includegraphics[width=.92\linewidth]{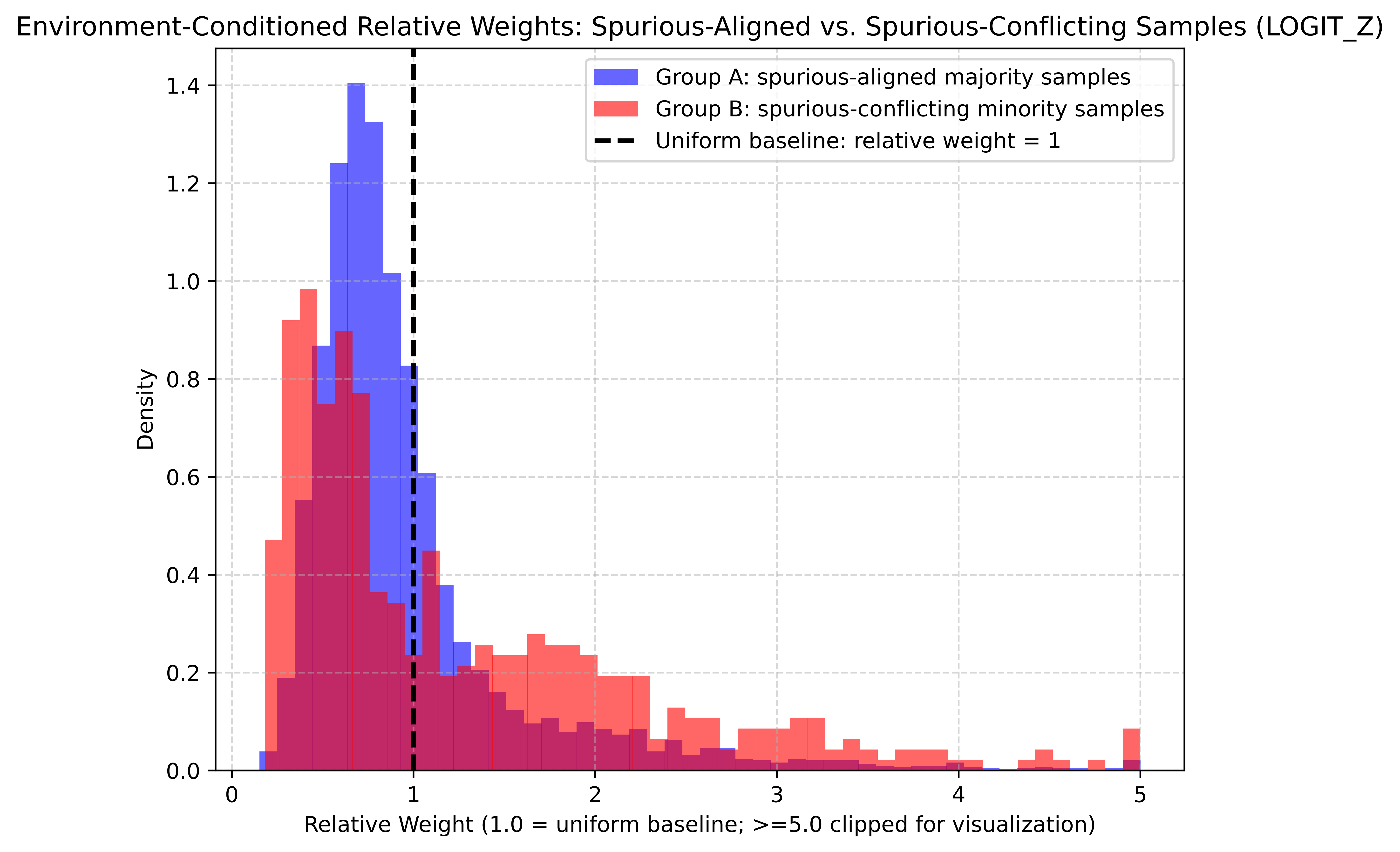}\\[6pt]
\includegraphics[width=.72\linewidth]{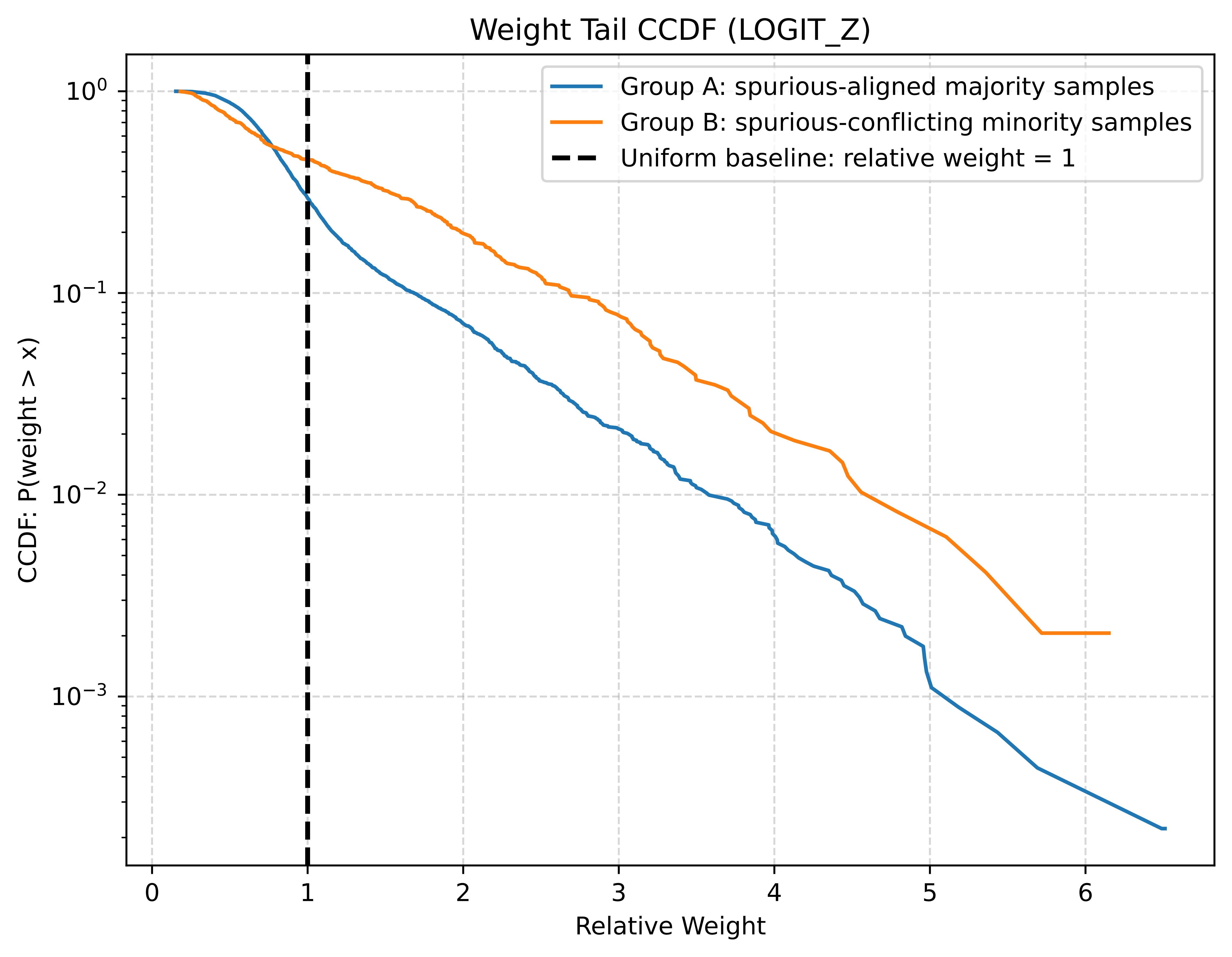}
\caption{Original evaluation, Easy mixed shift: conditional relative-weight density and CCDF. The dashed line is uniform weighting. Clipping affects the density display only.}
\label{fig:easy}
\end{figure}
\begin{figure}[htbp]
\centering
\includegraphics[width=.92\linewidth]{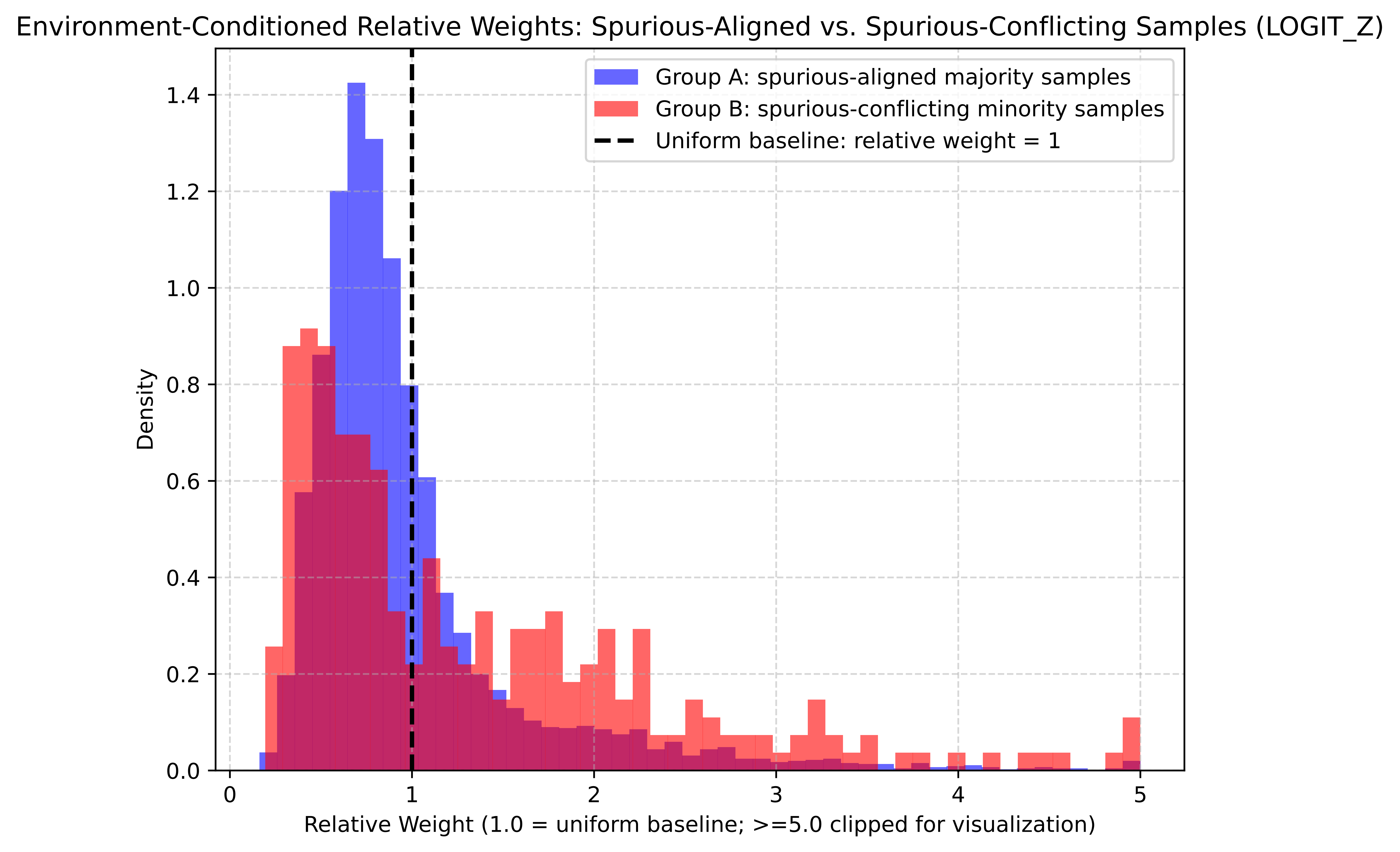}\\[6pt]
\includegraphics[width=.72\linewidth]{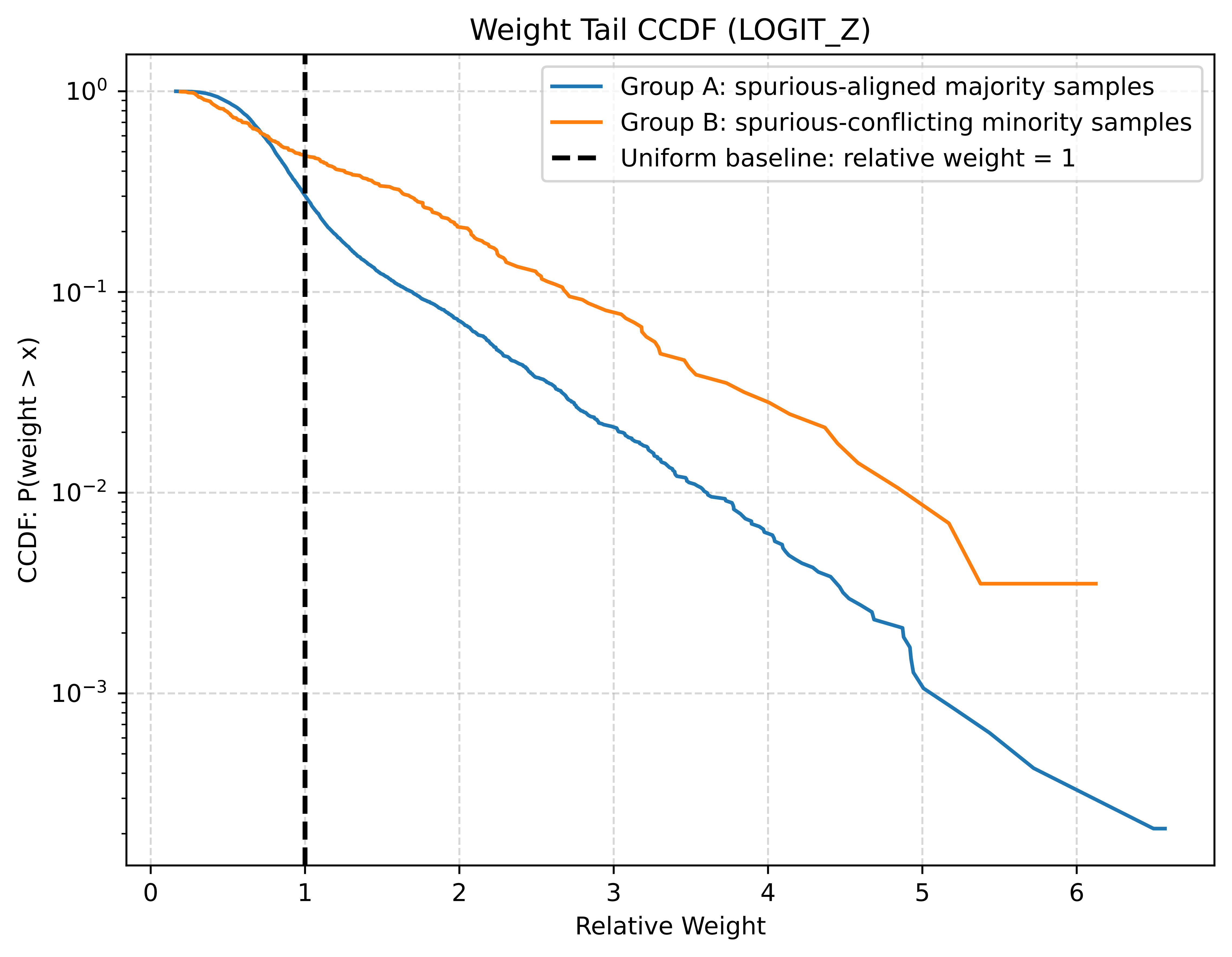}
\caption{Original evaluation, Medium mixed shift: conditional relative-weight density and CCDF. Minority samples have a heavier right tail, while many also receive weights below one.}
\label{fig:medium}
\end{figure}
\begin{figure}[htbp]
\centering
\includegraphics[width=.92\linewidth]{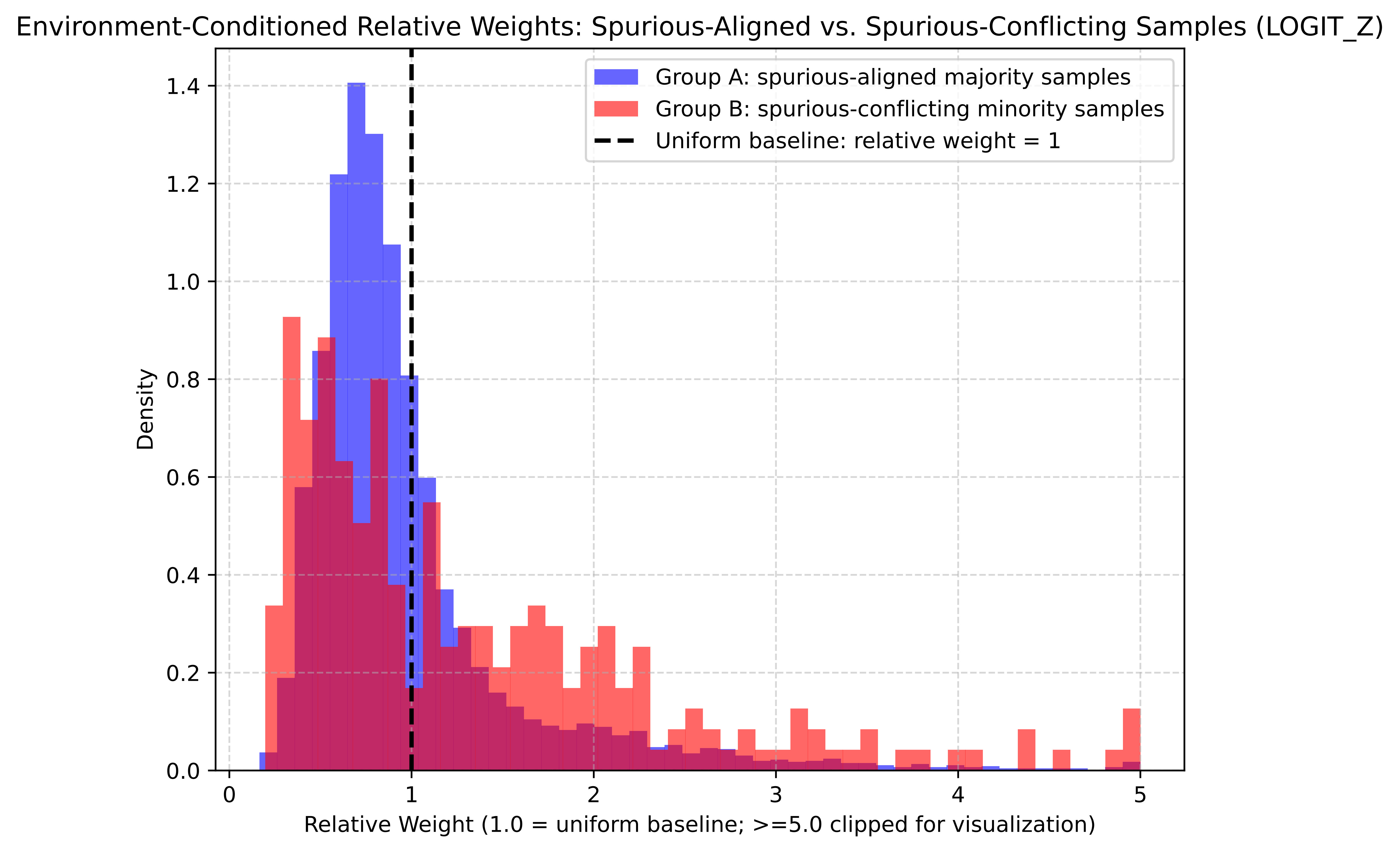}\\[6pt]
\includegraphics[width=.72\linewidth]{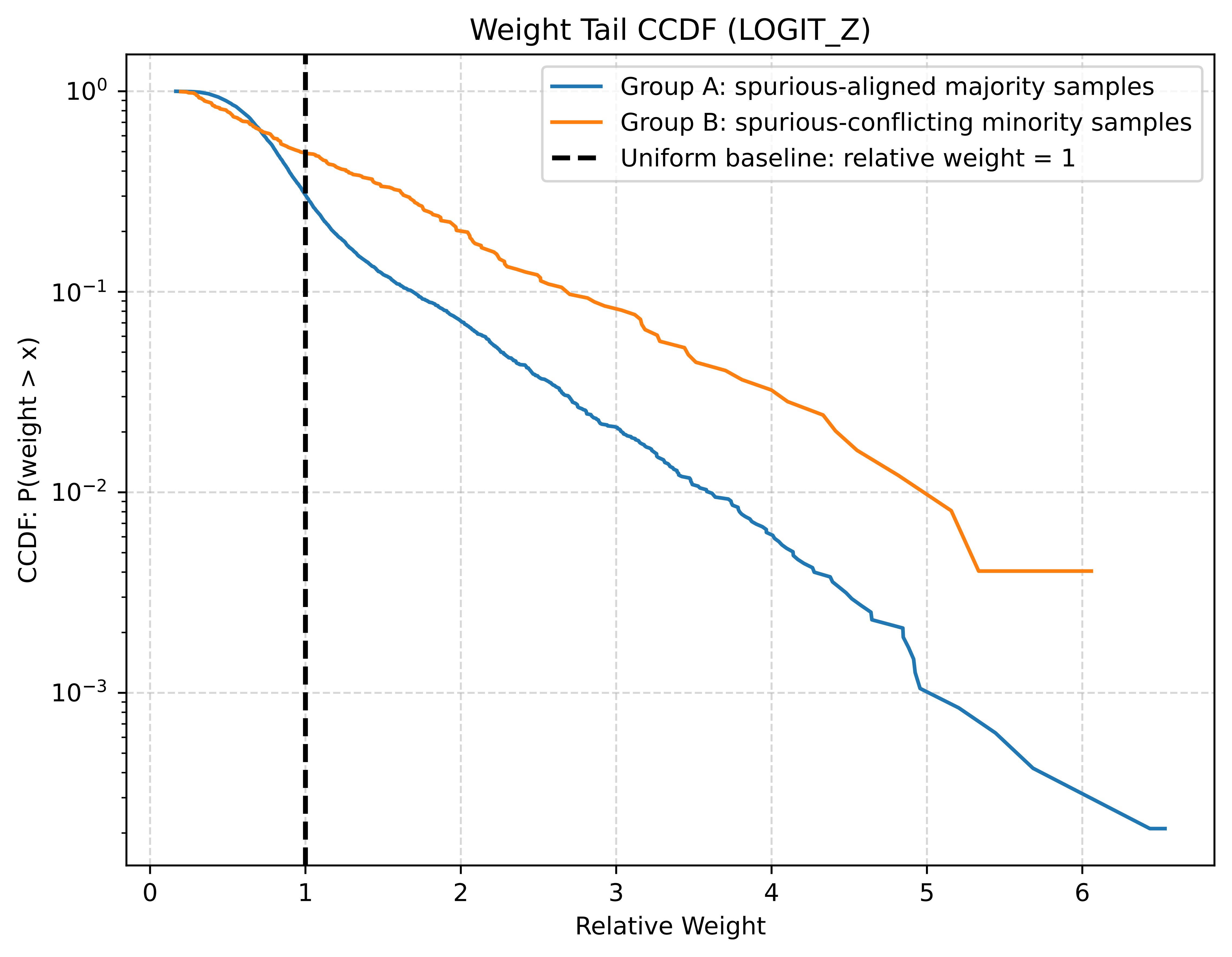}
\caption{Original evaluation, Hard mixed shift: conditional relative-weight density and CCDF. The minority distribution has a heavier right tail, illustrating selective emphasis on spurious-conflicting samples.}
\label{fig:hard}
\end{figure}
\FloatBarrier

\end{document}